\documentclass[conference]{IEEEtran}
\usepackage{times}

\usepackage[numbers]{natbib}
\usepackage{multicol}
\usepackage[bookmarks=true]{hyperref}
\usepackage{amsmath, amsfonts, leftidx, caption, subcaption, graphicx, algorithm, algorithmic, array, multirow}

\newtheorem{theorem}{Theorem}[section]

\newtheorem{lemma}[theorem]{Lemma}

\newcommand{\figureFolder}{figures}

\begin{document}

\title{Lyapunov-stable neural-network control}

\author{\authorblockN{Hongkai Dai \authorrefmark{1},
Benoit Landry\authorrefmark{2},
Lujie Yang\authorrefmark{3}, 
Marco Pavone\authorrefmark{2} and
Russ Tedrake\authorrefmark{1}\authorrefmark{3}}
\authorblockA{\authorrefmark{1} Toyota Research Institute}
\authorblockA{\authorrefmark{2}Stanford University}
\authorblockA{\authorrefmark{3}Massachusetts Institute of Technology\\ Email: hongkai.dai@tri.global,
blandry@stanford.edu
lujie@mit.edu
pavone@stanford.edu
russ.tedrake@tri.global}
}

\maketitle

\begin{abstract}
	Deep learning has had a far reaching impact in robotics. Specifically, deep reinforcement learning algorithms have been highly effective in synthesizing neural-network controllers for a wide range of tasks. However, despite this empirical success, these controllers still lack theoretical guarantees on their performance, such as Lyapunov stability  (i.e., all trajectories of the closed-loop system are guaranteed to converge to a goal state under the control policy). This is in stark contrast to traditional model-based controller design, where principled approaches (like LQR) can synthesize stable controllers with provable guarantees. To address this gap, we propose a generic method to synthesize a Lyapunov-stable neural-network controller, together with a neural-network Lyapunov function to simultaneously certify its stability. Our approach formulates the Lyapunov condition verification as a mixed-integer linear program (MIP). Our MIP verifier either certifies the Lyapunov condition, or generates counter examples that can help improve the candidate controller and the Lyapunov function. We also present an optimization program to compute an inner approximation of the  region of attraction for the closed-loop system. We apply our approach to robots including an inverted pendulum, a 2D and a 3D quadrotor, and showcase that our neural-network controller outperforms a baseline LQR controller. The code is open sourced at \url{https://github.com/StanfordASL/neural-network-lyapunov}.
\end{abstract}

\IEEEpeerreviewmaketitle

\section{Introduction}
The last few years have seen sweeping popularity of applying neural networks to a wide range of robotics problems \cite{sunderhauf2018limits}, such as perception \cite{kehl2017ssd, park2019deepsdf, florence2018dense}, reasoning \cite{fazeli2019see} and planning \cite{ichter2018learning}. In particular, researchers have had great success training control policies with neural networks on different robot platforms \cite{lee2020learning, tan2018sim, hwangbo2017control, kalashnikov2018qt}. Typically these control policies are obtained through reinforcement learning (RL) algorithms \cite{sutton2018reinforcement, schulman2015trust, haarnoja2018soft}. Although immensely successful, these neural-network controllers still generally lack theoretical guarantees on their performance, which could hinder their adoption in many safety-critical applications.

A crucial guarantee currently missing for neural-network  controllers is the stability of the closed-loop system, especially Lyapunov stability. A system is regionally stable in the sense of Lyapunov if starting from any states within a region, the system eventually converges to an equilibrium. This region is called the region of attraction (ROA) \cite{slotine1991applied}. Lyapunov stability provides a strong guarantee on the asymptotic behavior of the system for any state within the region of attraction. It is well known that a system is Lyapunov stable if and only if there exists a \textit{Lyapunov function} \cite{slotine1991applied} that is strictly positive definite and strictly decreasing everywhere except at the goal equilibrium state. Therefore, our goal is to synthesize a pair: a neural-network controller to stabilize the system, and a Lyapunov function to certify its stability.

\begin{figure}
    \centering
    \begin{subfigure}{0.26\textwidth}
    \includegraphics[width=0.98\textwidth]{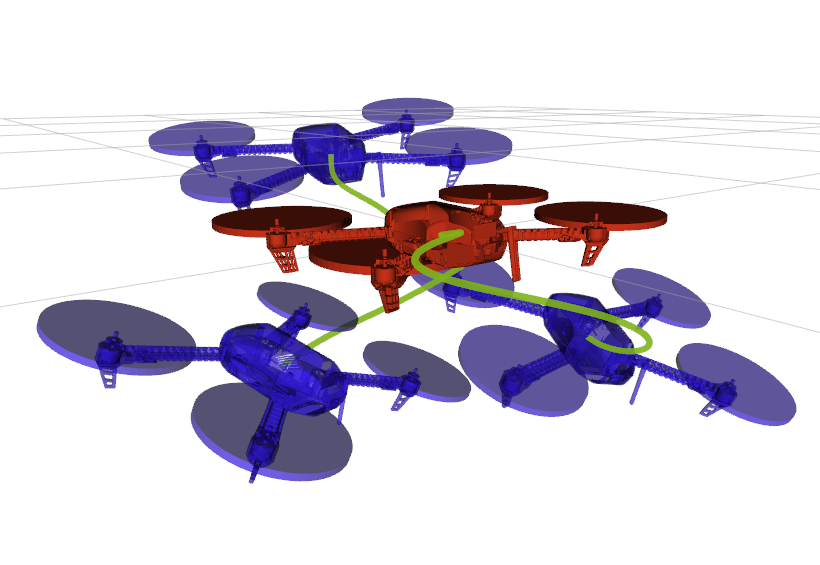}
    \end{subfigure}
    \begin{subfigure}{0.22\textwidth}
    \includegraphics[width=0.98\textwidth]{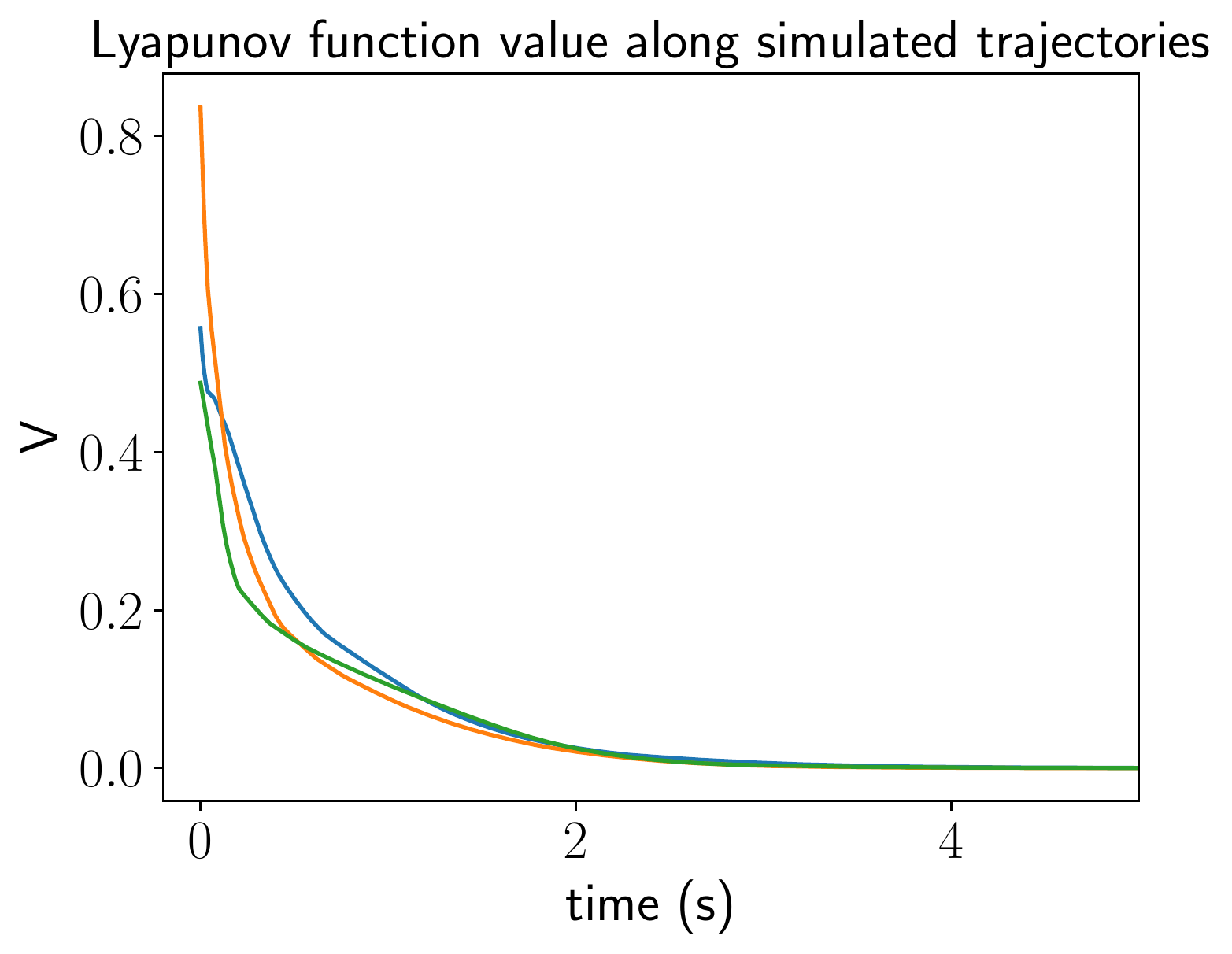}
    \end{subfigure}
    
    \caption{(left) Snapshots of stabilizing a 3D quadrotor with our neural-network controller to the hovering position at the origin (red snapshot) from different initial states. The green curves are the paths of the quadrotor center. (right) value of the neural-network Lyapunov functions along the simulated trajectories. The Lyapunov function has positive values, and decreases along the trajectories.}
    \label{fig:quadrotor3d_snapshot6}
\end{figure}
In the absence of neural networks in the loop, a significant body of work from the control community provides tools to synthesize Lyapunov-stable controllers \cite{slotine1991applied, boyd1994linear}. For example, for a linear dynamical system, one can synthesize a linear LQR controller to achieve Lyapunov stability (with the quadratic Lyapunov function solved through the Riccati equation). For a control-affine system with polynomial dynamics, Javis-Wloszek et al. \cite{jarvis2003some} and Majumdar et al. \cite{majumdar2013control} have demonstrated that a Lyapunov-stable controller together with a Lyapunov function, both polynomial functions of the state, can be obtained by solving a sum-of-squares (SOS) program. Recently, for more complicated systems, researchers have started to represent Lyapunov functions (but not their associated controllers) using neural networks. For example, Chang et al. synthesized linear controllers and neural network Lyapunov functions for simple nonlinear systems \cite{chang2019neural}. In a similar spirit, there is growing interest to approximate the system dynamics with neural networks, such as for racing cars \cite{williams2017information}, actuators with friction/stiction \cite{hwangbo2019learning}, perceptual measurement like keypoints \cite{manuelli2020keypoints}, system with contacts \cite{ pfrommer2020contactnets}, and soft robots \cite{gillespie2018learning},  where an accurate Lagrangian dynamics model is hard to obtain, while the neural-network dynamics model can be extracted from rich measurement data. Hence we are interested in systems whose dynamics are given as a neural network.

Unlike previous work which is restricted to linear \cite{boyd1994linear, chang2019neural} or polynomial controllers \cite{jarvis2003some, majumdar2013control}, our paper provides a novel approach to synthesize a stable neural-network controller, together with a neural-network Lyapunov function, for a given dynamical system whose forward dynamics is approximated by another neural network. The overall picture together with a Lyapunov function is visualized in Fig. \ref{fig:feedback_diagram}.
\begin{figure}
\begin{subfigure}{0.3\textwidth}
		\includegraphics[width=0.98\textwidth]{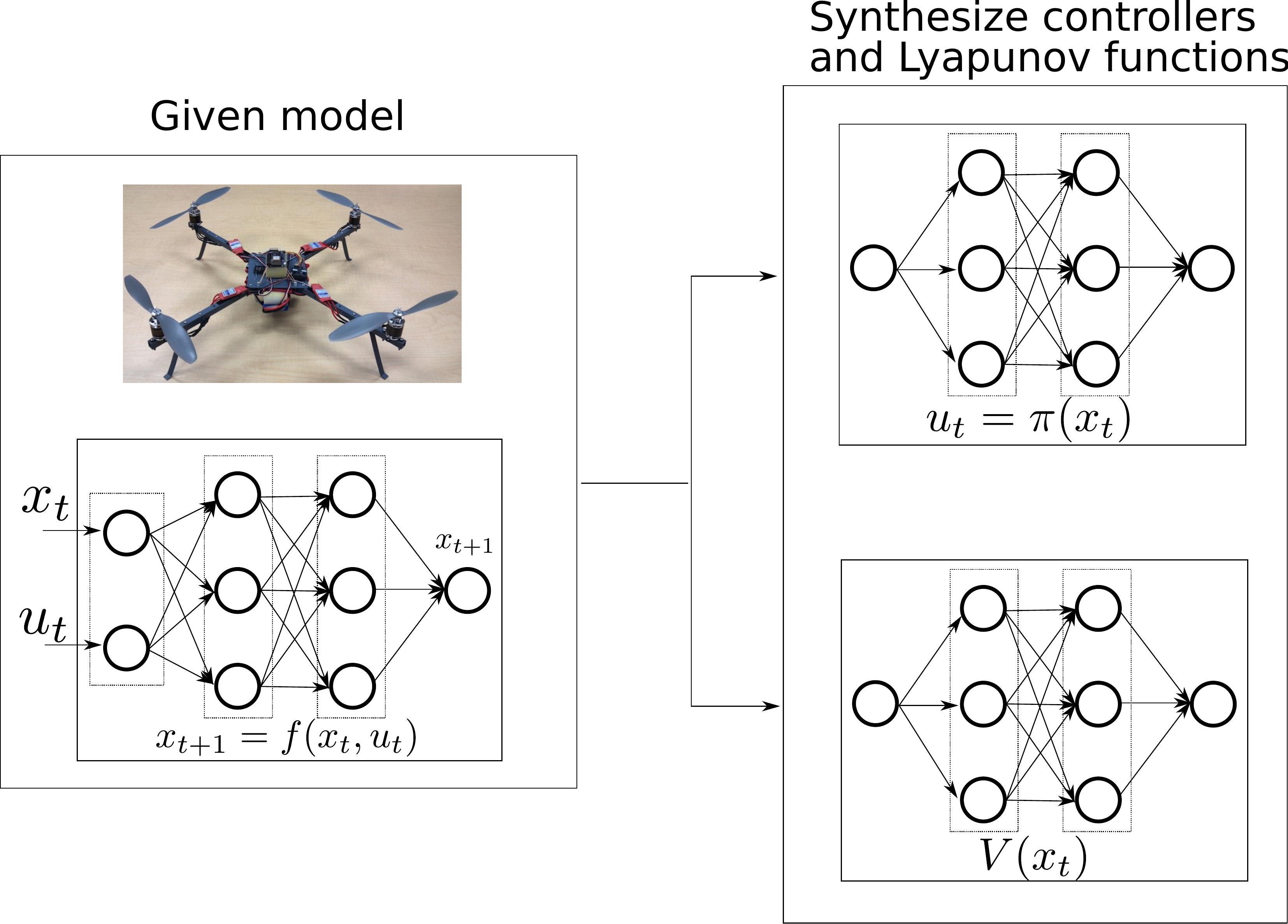}
		\subcaption{}
		\end{subfigure}
		\begin{subfigure}{0.17\textwidth}
		\includegraphics[width=0.98\textwidth]{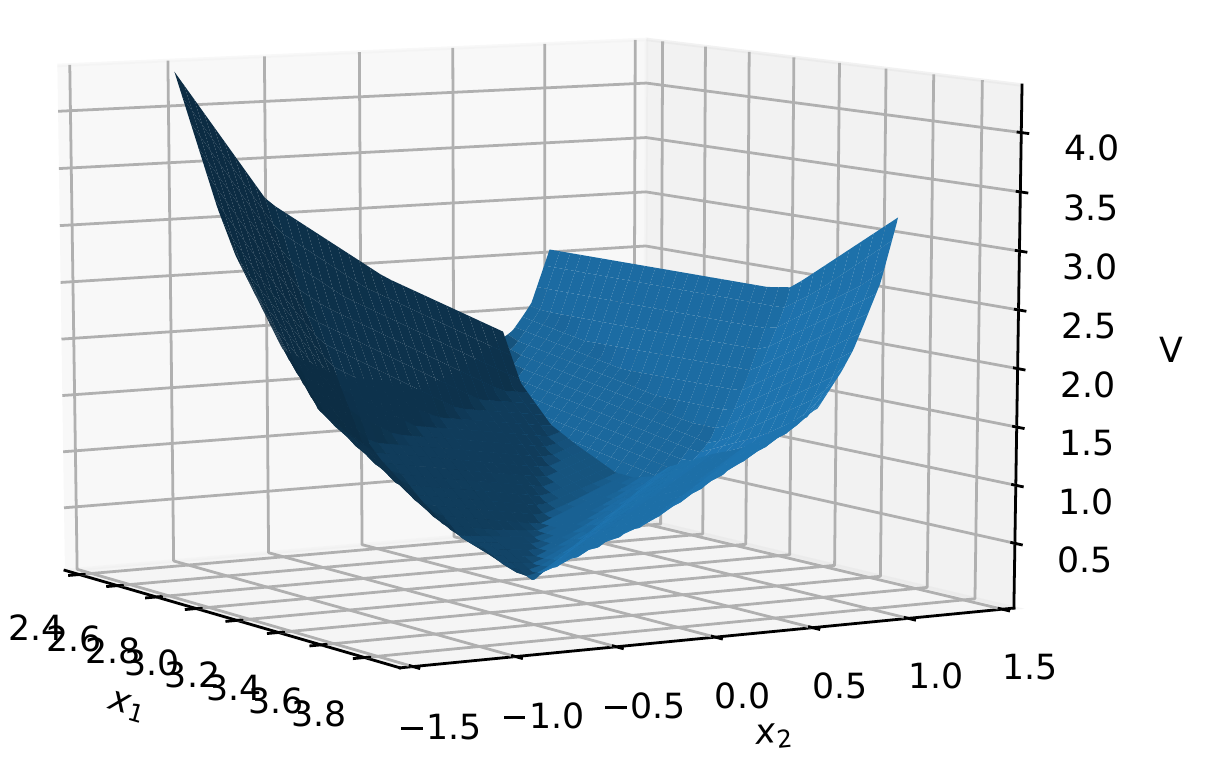}
		\subcaption{}
		\end{subfigure}
		\caption{(Left) The forward system\protect\footnotemark (which contains a neural network) is given, and we aim at finding the controller and a Lyapunov function to prove Lyapunov stability of the closed-loop system. Both the controller and Lyapunov function contain neural networks. (right) Visualization of a Lyapunov function for a 2-dimensional system. The Lyapunov function is usually a bowl-shaped function that is strictly positive except at the goal state.}
	\label{fig:feedback_diagram}
\end{figure}
\footnotetext{The quadrotor picture is taken from \cite{bonna2015trajectory}.}

In order to synthesize neural-network controllers and Lyapunov functions, one has to first be able to verify that the neural-network functions satisfy the Lyapunov condition for \textit{all} states within a region. There are several techniques to verify certain properties of neural network outputs for all inputs within a range. These techniques can be categorized by whether the verification is exact, e.g., using Satisfiability Modulo Theories (SMT) solvers \cite{katz2017reluplex, chang2019neural, abate2020formal} or mixed-integer programs (MIP) solvers \cite{bunel2018unified, tjeng2018evaluating, dai2020counter, chen2020learning}, versus inexact verification by solving a relaxed convex problem \cite{bastani2016measuring, fazlyab2020safety, wong2018provable, Shen2020}. Another important distinction among these techniques is the activation functions used in the neural networks. For example, Abate et al. \cite{abate2020formal} and Chang et al. \cite{chang2019neural} learn neural-network Lyapunov functions with quadratic and \textit{tanh} activation functions respectively. On the other hand, the piecewise linear nature of (leaky) ReLU activation implies that  the input and output of a (leaky) ReLU network satisfy mixed-integer linear constraints, and hence network properties can be exactly verified by MIP solvers \cite{tjeng2018evaluating, dai2020counter}. In this work, due to its widespread use, we choose the (leaky) ReLU unit for all neural networks. This enables us to perform exact verification of the Lyapunov condition without relaxation for safety-critical robot missions.

The verifiers (both SMT and MIP solvers) can either definitively certify that a given candidate function satisfies the Lyapunov condition everywhere in the region, or generate counter examples violating the Lyapunov condition. In this work, we solve MIPs to find the most adversarial counter examples, namely the states with the maximal violation of the Lyapunov condition. Then, in order to improve the satisfaction of the Lyapunov conditions, we propose two approaches to jointly train the controller and the Lyapunov function. The first approach is a standard procedure in counter-example guided training, where we add the counter examples to the training set and minimize a surrogate loss function of the Lyapunov condition violation on this training set \cite{abate2020formal, chang2019neural, ravanbakhsh2019learning}. The second approach is inspired by the bi-level optimization community \cite{bard2013practical, landry2019differentiable, dai2020counter}, where we directly minimize the maximal violation as a min-max problem through gradient descent. 

Our contributions include: 1) we synthesize a Lyapunov-stable neural-network controller together with a neural-network Lyapunov function. To the best of our knowledge, this is the first work capable of doing this. 2) We compute an inner approximation of the  region of attraction for the closed-loop system. 3) We present two approaches to improve the networks based on the counter examples found by the MIP verifier. 4) We demonstrate that our approach can successfully synthesize Lyapunov-stable neural-network controllers for systems including inverted pendulums, 2D and 3D quadrotors, and that they outperform a baseline LQR controller.

\section{Problem statement}
\label{sec:problem_definition}
We consider a discrete-time system whose forward dynamics is
\begin{subequations}
\begin{align}
	x_{t+1} = f(x_t, u_t) = \phi_{\text{dyn}}(x_t, u_t) - \phi_{\text{dyn}}(x^*, u^*) + x^*\label{eq:forward_dynamics}\\
	u_{\text{min}}\le u_t \le u_{\text{max}}\label{eq:input_limits}
\end{align}
\end{subequations}
where $x_t\in\mathbb{R}^{n_x}, u_t\in\mathbb{R}^{n_u}$, $u_{\text{min}}$ and $u_{\text{max}}$ are the lower/upper input limits. $\phi_{\text{dyn}}$ is a feedforward fully connected neural network with leaky ReLU activation functions \footnote{Since ReLU can be regarded as a special case of leaky ReLU, we present our work with leaky ReLU for generality.} \footnote{Our approach can also handle other architectures such as convolution. For simplicity of presentation we don't discuss them in this paper.}. $x^*/u^*$ are the state/control at the goal equilibrium. By definition the dynamics equation \eqref{eq:forward_dynamics} guarantees that at the equilibrium state/control $x_t = x^*, u_t = u^*$, the next state $x_{t+1}$ remains the equilibrium state $x_{t+1}=x^*$. Due to the universal approximation theorem \cite{leshno1993multilayer}, we can approximate an arbitrary smooth dynamical system written as \eqref{eq:forward_dynamics} with a neural network. Our goal is to find a control policy $u_t = \pi(x_t)$ and a Lyapunov function $V(x_t):\mathbb{R}^{n_x}\rightarrow \mathbb{R}$, such that the following Lyapunov conditions are satisfied:
\begin{subequations}
\begin{align}
	V(x_t) > 0 \;\forall x_t\in\mathcal{S}, x_t\ne x^*\label{eq:lyapunov_positivity}\\
	V(x_{t+1}) - V(x_t) \le -\epsilon_2 V(x_t)\;\forall x_t\in\mathcal{S}, x_t\ne x^*\label{eq:lyapunov_derivative}\\
	V(x^*) = 0\label{eq:lyapunov_equlibrium}
\end{align}
\label{eq:lyapunov_condition}%
\end{subequations}
where $\mathcal{S}$ is a compact sub-level set $\mathcal{S} = \{x_t | V(x_t)\le \rho\}$, and $\epsilon_2 > 0$ is a given positive scalar. The Lyapunov conditions in \eqref{eq:lyapunov_condition} guarantee that starting from any state inside $\mathcal{S}$, the state converges exponentially to the equilibrium state $x^*$, and $\mathcal{S}$ is a region of attraction of the closed-loop system. In addition to the control policy and the Lyapunov function, we will find an inner approximation of the region of attraction.  Note that condition \eqref{eq:lyapunov_derivative} is a constraint on the Lyapunov function $V(\cdot)$ as well as the control policy $\pi(\cdot)$, since $V(x_{t+1}) = V(f(x_t, \pi(x_t)))$ depends on both the control policy to compute $x_{t+1}$ together with the Lyapunov function $V(\cdot)$.

\section{Background on ReLU and MIP}
\label{sec:background}
In this section we give a brief overview of the mixed-integer linear formulation which encodes the input/output relationship of a neural network with leaky ReLU activation. This MIP formulation arises from the network output being a piecewise-affine function of the input, hence intuitively one can use linear constraints for each affine piece, and binary variables for the activated piece. Previously researchers have solved mixed-integer programs (MIP) to verify certain properties of the feedforward neural network in machine learning applications such as verifying image classifiers \cite{bunel2018unified, tjeng2018evaluating}. 
\begin{figure}
\centering
		\includegraphics[width=0.15\textwidth]{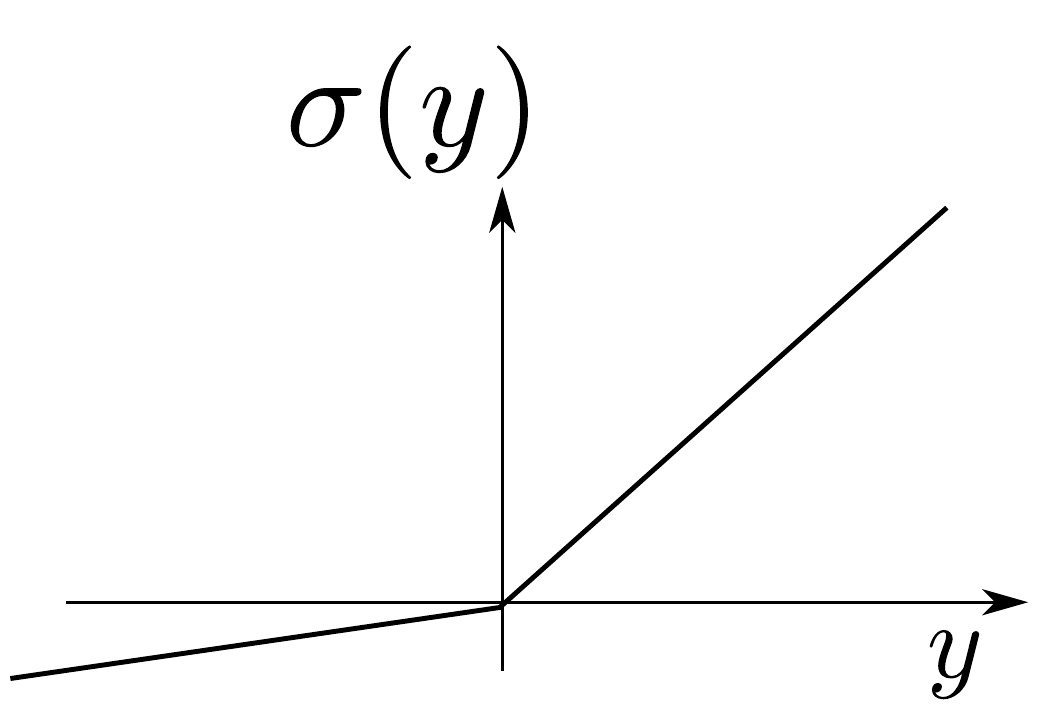}
		\caption{A leaky ReLU activation function.}
		\label{fig:leaky_relu}
\end{figure}

For a general fully-connected neural network, the input/output relationship in each layer is
\begin{subequations}
\begin{align}
	z_i =& \sigma(W_{i}z_{i-1}+b_{i}), i=1,\hdots,n-1\\
	z_n =& W_{n}z_{n-1}+b_{n},
	z_0 = x,
\end{align}
\label{eq:fully_connected_network}%
\end{subequations}
where $W_{i}, b_{i}$ are the weights/biases of the i'th layer. The activation function $\sigma(\cdot)$ is the leaky ReLU function shown in Fig.\ref{fig:leaky_relu}, as a piecewise linear function $\sigma(y) = \max(y, cy)$ where $0 \le c<1$ is a given scalar. If we suppose that for one leaky ReLU neuron, the input $y\in\mathbb{R}$ is bounded in the range $y_{\text{lo}} \le y \le y_{\text{up}}$ (where $y_{\text{lo}}< 0$ and $y_{\text{up}}>0$), then we can use the big-M technique to write out the input/output relationship of a leaky ReLU unit $w=\sigma(y)$ as the following mixed-integer linear constraints
\begin{subequations}
\begin{align}
	w\ge y,\quad
	w \ge cy\\
	w \le cy - (c-1)y_{\text{up}}\beta,\quad
	w \le y - (c-1)y_{\text{lo}}(\beta-1)\\
	\beta\in \{0, 1\},
\end{align}
\label{eq:leaky_relu_mix_integer}%
\end{subequations}
where the binary variable $\beta$ is active when $y\ge 0$. Since the only nonlinearity in the neural network \eqref{eq:fully_connected_network} is the leaky ReLU unit $\sigma(\cdot)$, by replacing it with constraints \eqref{eq:leaky_relu_mix_integer}, the relationship between the network output $z_n$ and input $x$ is fully captured by mixed-integer linear constraints.

We expect \textit{bounded} input to the neural networks since we care about states within a neighbourhood of the equilibrium so as to prove regional Lyapunov stability, and the system input $u_t$ is restricted within the input limits (Eq. \eqref{eq:input_limits}). With a bounded neural network input, the bound of each ReLU neuron input can be computed by either \textit{Interval Arithmetic} \cite{wong2018provable}, by solving a linear programming (LP) problem \cite{tjeng2018evaluating}, or by solving a mixed-integer linear programming (MILP) problem \cite{Cheng2017maximum, Fischetti2018}.

After formulating neural network verification as a mixed-integer program (MIP), we can efficiently solve MIPs to global optimality with off-the-shelf solvers, such as Gurobi \cite{gurobi} and CBC \cite{forrest2005cbc} via branch-and-cut method.

\section{Approach}
\label{sec:approach}
In this section we present our approach to finding a pair of neural networks as controller and the Lyapunov function. We will first use the technique described in the previous section \ref{sec:background}, and demonstrate that one can verify the Lyapunov condition \eqref{eq:lyapunov_condition} through solving MIPs. Then we will present two approaches to reduce the Lyapunov condition violation using the MIP results. Finally we explain how to compute an inner-approximation of the region of attraction.

\subsection{Verify Lyapunov condition via solving MIPs}
\label{subsec:lyapunov_mip}
We represent the Lyapunov function with a neural network $\phi_V:\mathbb{R}^{n_x}\rightarrow\mathbb{R}$ as
\begin{align}
	V(x_t) = \phi_V(x_t) - \phi_V(x^*) + |R(x_t - x^*)|_1, \label{eq:nn_lyapunov}
\end{align}
where $R$ is a matrix with full column rank. $|R(x_t - x^*)|_1$ is the 1-norm of the vector $R(x_t-x^*)$. Eq. \eqref{eq:nn_lyapunov} guarantees $V(x^*)=0$, hence condition \eqref{eq:lyapunov_equlibrium} is trivially satisfied. Notice that even without the term $|R(x_t - x^*)|_1$ in \eqref{eq:nn_lyapunov}, the Lyapunov function would still satisfy $V(x^*) = 0$, but adding this 1-norm term assists $V(\cdot)$ in satisfying the Lyapunov condition $V(x_t)>0$. As visualized in Fig \ref{fig:lyapunov_add_l1_3}, $\phi_V(x_t) - \phi_V(x^*)$ is a piecewise-affine function of $x_t$ passing through the point $(x^*, 0)$. Most likely $(x^*, 0)$ is in the interior of one of the linear pieces, instead of on the boundary of a piece; hence locally around $x^*$, the term $\phi_V(x_t) - \phi_V(x^*)$ is a linear function of $x_t$, which will become negative away from $x^*$, violating the positivity condition $V(x_t)>0$ (\eqref{eq:lyapunov_positivity}). To remedy this, we add the term $|R(x_t - x^*)|_1$ to the Lyapunov function. Due to $R$ being full-rank, this 1-norm is strictly positive everywhere except at $x^*$. With sufficiently large $R$, we guarantee that  at least locally around $x^*$ the Lyapunov function is positive. Notice that $V(x_t)$ is a piecewise-affine function of $x_t$.
\begin{figure}
	\centering
	\includegraphics[width=0.4\textwidth]{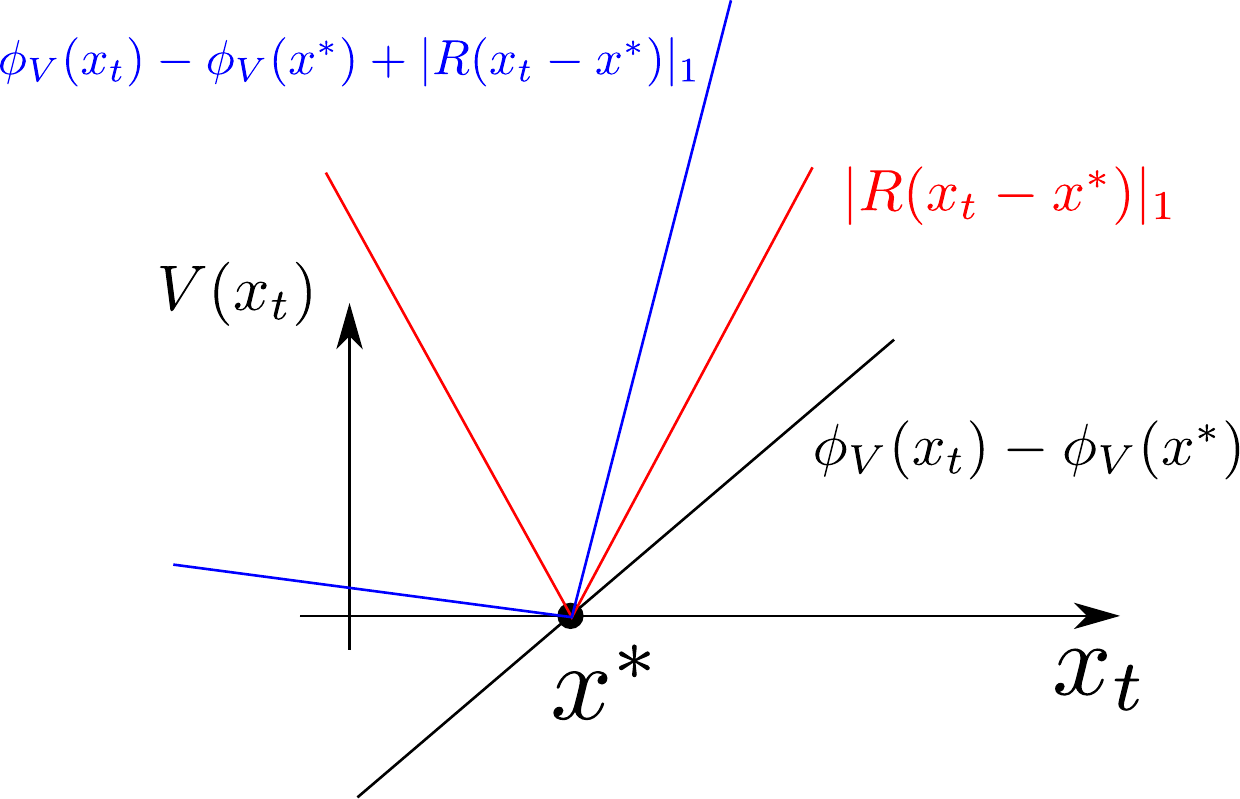}
	\caption{The term $\phi_V(x_t) - \phi_V(x^*)$ is a piecewise-affine function that passes through $(x_t, V(x_t)) = (x^*, 0)$. Most likely $x^*$ is within the interior of a linear piece, but not at the boundary between pieces. This linear piece will go negative in the neighbourhood of $x^*$. By adding the 1-norm $|R(x_t-x^*)|_1$ (red lines), the Lyapunov function (blue lines) is at least locally positive around $x^*$.}
	\label{fig:lyapunov_add_l1_3}
\end{figure}

Our approach will entail searching for both the neural network $\phi_V$ and the full column-rank matrix $R$ in \eqref{eq:nn_lyapunov}. To guarantee $R$ being full column-rank, we parameterize it as
\begin{align}
	R = U \left(\Sigma + \text{diag}(r_1^2, \hdots, r_{n_x}^2)\right)V^T, \label{eq:R_svd}
\end{align} where $U, V$ are given orthonormal matrices, $\Sigma$ is a given diagonal matrix with strictly positive diagonal entries, and scalars $r_1, \hdots, r_{n_x}$ are free variables. The parameterization \eqref{eq:R_svd} guarantees $R$ being full column-rank since the minimal singular value of $R$ is strictly positive.

We represent the control policy using a neural network $\phi_\pi:\mathbb{R}^{n_x}\rightarrow\mathbb{R}^{n_u}$ as
\begin{align}
	u_t = \pi(x_t) = \text{clamp}\left(\phi_\pi(x_t) - \phi_\pi(x^*) + u^*, u_{\text{min}}, u_{\text{max}}\right), \label{eq:control_network}
\end{align}
where $\text{clamp}(\cdot)$ clamps the value $\phi_\pi(x_t) - \phi_\pi(x^*) + u^*$ elementwisely within the input limits $[u_{\text{min}}, u_{\text{max}}]$, namely
\begin{align}
	\text{clamp}(\alpha, \text{lo}, \text{up}) = \begin{cases} \text{up} \;\text{if } \alpha > \text{up}\\ \alpha \;\text{if } \text{lo}\le \alpha \le \text{up}\\ \text{lo}\; \text{if }\alpha < \text{lo}\end{cases}. \label{eq:clamp_u}
\end{align}

The control policy \eqref{eq:control_network} is a piecewise-affine function of the state $x_t$. Notice that \eqref{eq:control_network} guarantees that at the equilibrium state $x_t = x^*$, the control action is $u_t = u^*$.

It is worth noting that our approach is only applicable to systems that can be stabilized by \textit{regular} (e.g., locally Lipschitz bounded) controllers. Some dynamical systems, for example a unicycle, require non-regular controllers for stabilization, where our approach would fail. The readers can refer to \cite{sontag1999stability} for more background on regular controllers.

The Lyapunov condition \eqref{eq:lyapunov_condition}, in particular, \eqref{eq:lyapunov_positivity}, is a strict inequality. To verify this through MIP which only handles non-strict inequalities constraints $\ge$ and $\le$, we change condition \eqref{eq:lyapunov_positivity} to the following condition with $\ge$
\begin{align}
	V(x_t) \ge \epsilon_1 |R(x_t-x^*)|_1\;\forall x\in\mathcal{S} \label{eq:lyapunov_positivity_l1},
\end{align}
where $0<\epsilon_1<1$ is a given positive scalar. Since $R$ is full column-rank, the right-hand side is 0 only when $x_t=x^*$. Hence the non-strict inequality constraint \eqref{eq:lyapunov_positivity_l1} is a sufficient condition for the strict inequality constraint \eqref{eq:lyapunov_positivity}. In Appendix \ref{subsec:l1_necessary} we prove that it is also a necessary condition.

In order to verify the Lyapunov condition \eqref{eq:lyapunov_condition} for a candidate Lyapunov function and a controller, we consider verifying the condition \eqref{eq:lyapunov_positivity_l1} and \eqref{eq:lyapunov_derivative} for a given bounded polytope $\mathcal{B}$ around the equilibrium state. The verifier solves the following optimization problems
\begin{subequations}
\begin{align}
	\max_{x_t\in\mathcal{B}} \epsilon_1|R(x_t-x^*)|_1 - V(x_t) \label{eq:lyapunov_positivity_mip}\\
	\max_{x_t\in\mathcal{B}} V(x_{t+1}) - V(x_t) + \epsilon_2 V(x_t),\label{eq:lyapunov_derivative_mip}
\end{align}
\label{eq:lyapunov_mip}%
\end{subequations}
where the objectives are the violation of condition \eqref{eq:lyapunov_positivity_l1} and \eqref{eq:lyapunov_derivative} respectively. If the optimal values of both problems are 0 (attained at $x_t=x^*$), then we certify the Lyapunov condition \eqref{eq:lyapunov_condition}. The objective in \eqref{eq:lyapunov_positivity_mip} is a piecewise-affine function of the variable $x_t$ since both $V(x_t)$ and $|R(x_t-x^*)|_1$ are piecewise-affine. Likewise in optimization problem \eqref{eq:lyapunov_derivative_mip}, since the control policy \eqref{eq:control_network} is a piecewise-affine functions of $x_t$, and the forward dynamics \eqref{eq:forward_dynamics} is a piecewise-affine function of $x_t$ and $u_t$, the next state $x_{t+1}=f(x_t, \pi(x_t))$, its Lyapunov value $V(x_{t+1})$ and eventually the objective function in \eqref{eq:lyapunov_derivative_mip} are all piecewise-affine functions of $x_t$. It is well known in the optimization community that one can maximize a piecewise-affine function within a bounded domain ($\mathcal{B}$ in this case) through solving an MIP \cite{vielma2010mixed}. In section \ref{sec:background} we have shown the MIP formulation on neural networks with leaky ReLU units; in Appendix \ref{subsec:l1_clamp_MIP} we present the MIP formulation for the 1-norm in $|R(x_t-x^*)|_1$ and the clamp function in the control policy.

By solving the mixed-integer programs in \eqref{eq:lyapunov_mip}, we either verify that the candidate controller is Lyapunov-stable with the candidate Lyapunov function $V(x_t)$ as a stability certificate; or we generate counter examples of $x_t$, where the objective values are positive, hence falsify the candidates. By maximizing the Lyapunov condition violation in the MIP \eqref{eq:lyapunov_mip}, we find not only \textit{a} counter example if one exists, but the \textit{worst} counter example with the largest violation. Moreover, since the MIP solver traverses a binary tree during branch-and-cut, where each node of the tree might find a counter example, the solver finds a list of counter examples during the solving process. In the next subsection, we use both the worst counter example and the list of all counter examples to reduce the Lyapunov violation.

\subsection{Trainer}
\label{subsec:trainer}
After the MIP verifier generates counter examples violating Lyapunov conditions, to reduce the violation, we use these counter examples to improve the candidate control policy and the candidate Lyapunov function. We present two iterative approaches. The first one minimizes a surrogate function on a training set, and the counter examples are appended to the training set in each iteration. This technique is widely used in the counter-example guided training \cite{chang2019neural, abate2020formal, chen2020learning, kapinski2014simulation}. The second approach minimizes the maximal Lyapunov condition violation directly by solving a min-max problem through gradient descent. In both approaches, we denote the parameters we search for as $\theta$, including
\begin{itemize}
	\item The weights/biases in the controller network $\phi_\pi$;
	\item The weights/biases in the Lyapunov network $\phi_V$;
	\item $r_1, \hdots, r_{n_x}$ in the full column-rank matrix $R$ (Eq. \eqref{eq:R_svd}).
\end{itemize}
namely we optimize both the control policy and the Lyapunov function simultaneously, so as to satisfy the Lyapunov condition on the closed-loop system.

\subsubsection{Approach 1, growing training set with counter examples}
\label{subsubsec:train_counter_examples}
A necessary condition for satisfying the Lyapunov condition for \text{any} state in $\mathcal{B}$, is that the Lyapunov condition holds for many sampled states within $\mathcal{B}$. Hence we could reduce a surrogate loss function on a training set $\mathcal{X}$ containing sampled states. The training set $\mathcal{X}$ grows after each MIP solve by appending the counter examples generated from the MIP solve. Since the MIP \eqref{eq:lyapunov_positivity_mip} and the MIP \eqref{eq:lyapunov_derivative_mip} generate different counter examples, we keep two separate training sets $\mathcal{X}_1$ and $\mathcal{X}_2$ for MIP \eqref{eq:lyapunov_positivity_mip} and MIP \eqref{eq:lyapunov_derivative_mip} respectively.

We design a surrogate loss function for $\mathcal{X}_1, \mathcal{X}_2$ to measure the violation of the Lyapunov condition on the training set. We denote the violation of condition \eqref{eq:lyapunov_positivity_l1} on a sample state $x^{i}_1\in\mathcal{X}_1$ as $\eta_1(x^{i}_1)$, and the violation of condition \eqref{eq:lyapunov_derivative} on a sample state $x^{i}_2 \in\mathcal{X}_2$ as $\eta_2(x^{i}_2)$, defined as
\begin{subequations}
\begin{align}
	\eta_1(x^{i}_1) =& \max(\epsilon_1|R(x^{i}_1-x^*)|_1 - V(x^{i}_1), 0)\label{eq:violation_1}\\
	\eta_2(x^{i}_2) =& \max(V(f(x^{i}_2, \pi(x^{i}_2))) - V(x^{i}_2) + \epsilon_2V(x^{i}_2), 0),
\end{align}
\end{subequations}
We denote $\eta_1(\mathcal{X}_1)$ and $\eta_2(\mathcal{X}_2)$ as the vectors whose i'th entry is the violation on the i'th sample $\eta_1(x^{i}_1)$ and $\eta_2(x^i_2)$ respectively, then our surrogate function is defined as
\begin{align}
	\text{loss}_\theta(\mathcal{X}_1, \mathcal{X}_2) = |\eta_1(\mathcal{X}_1)|_p + |\eta_2(\mathcal{X}_2)|_p, \label{eq:training_set_loss}
\end{align}
where $|\cdot|_p$ denotes the p-norm of a vector, such as 1-norm (mean of the violation), $\infty-$norm (maximal of the violation) and $4-$norm (a smooth approximation of the $\infty-$ norm). The subscript $\theta$ in the loss function \eqref{eq:training_set_loss} emphasizes its dependency on $\theta$, the parameters in both the controller and the Lyapunov function. We then minimize the surrogate loss function on the training set via standard batched gradient descent on $\theta$. The flow chart of this approach is depicted in Fig. \ref{fig:training_flow_chart}. Algorithm \ref{algorithm:train_on_sample} presents the pseudo-code.
\begin{figure}
	\centering
	\includegraphics[width=0.4\textwidth]{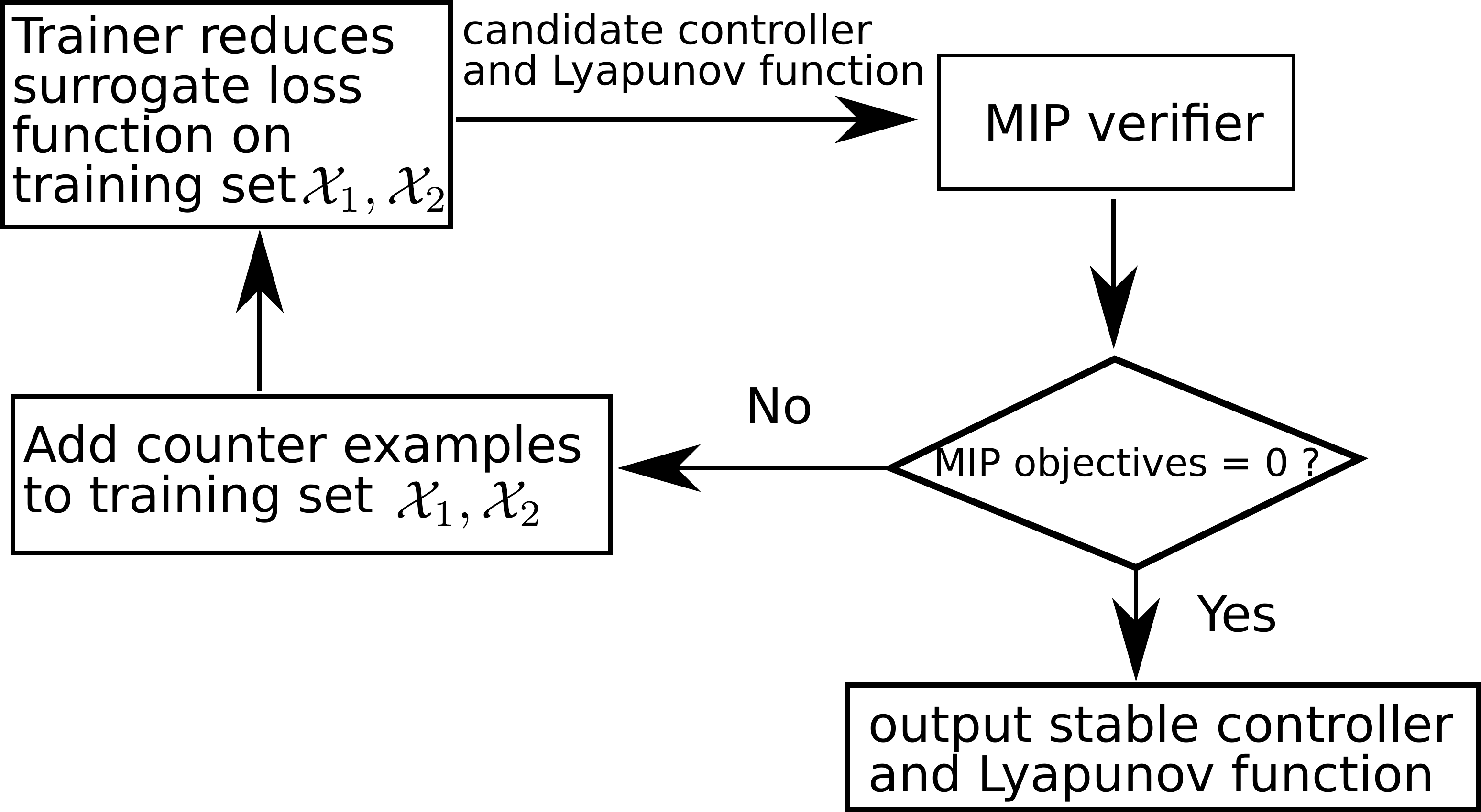}
	\caption{Flow chart of Algorithm \ref{algorithm:train_on_sample}.}
	\label{fig:training_flow_chart}
\end{figure}

\begin{algorithm}
\caption{Train controller/Lyapunov function on training sets constructed from verifier}
\label{algorithm:train_on_sample}
\begin{algorithmic}[1]
	\STATE Start with a candidate neural-network controller $\pi$, a candidate Lyapunov function $V$, and training sets $\mathcal{X}_1, \mathcal{X}_2$.
	\WHILE{not converged}
	\STATE Solve MIPs \eqref{eq:lyapunov_positivity_mip} and \eqref{eq:lyapunov_derivative_mip}.
	\IF{MIP \eqref{eq:lyapunov_positivity_mip} or MIP \eqref{eq:lyapunov_derivative_mip} has maximal objective $ > 0$}
		\IF{MIP \eqref{eq:lyapunov_positivity_mip} maximal objective $> 0$}
		\STATE Add the counter examples from MIP \eqref{eq:lyapunov_positivity_mip} to $\mathcal{X}_1$.
		\ENDIF
		\IF{MIP \eqref{eq:lyapunov_derivative_mip} maximal objective $>0$}
		\STATE Add the counter examples from MIP \eqref{eq:lyapunov_derivative_mip} to $\mathcal{X}_2$.
		\ENDIF
		\STATE Perform batched gradient descent on the parameters $\theta$ to reduce the loss function \eqref{eq:training_set_loss} on the training set $\mathcal{X}_1, \mathcal{X}_2$. Stop until either $\text{loss}_{\theta}(\mathcal{X}_1, \mathcal{X}_2)=0$, or reaches a maximal epochs.
	\ELSE
	\STATE converged = true.
	\ENDIF
	\ENDWHILE
\end{algorithmic}
\end{algorithm}
Since the surrogate loss function is the Lyapunov condition violation on just the sampled states, the batched gradient descent will overfit to the training set, and potentially cause large violation away from the sampled states. To avoid this overfitting problem, we consider an alternative approach without constructing the training sets.
\subsubsection{Approach 2, minimize the violation via min-max program}
\label{subsubsec:bilevel_approach}
Instead of minimizing a surrogate loss function on a training set, we can minimize the Lyapunov condition violation directly through the following min-max problem
\begin{equation}
	\begin{split}
		\min_{\theta}\left(\underbrace{\max_{x_t\in\mathcal{B}} \epsilon_1|R(x_t-x^*)|_1-V(x_t)}_{\text{MIP } \eqref{eq:lyapunov_positivity_mip}} \right.\qquad\qquad\\
				   \left.+ \underbrace{\max_{x_t\in\mathcal{B}}V(x_{t+1}) - V(x_t) + \epsilon_2 V(x_t)}_{\text{MIP }\eqref{eq:lyapunov_derivative_mip}}\right),\label{eq:bilevel_optimization}
\end{split}
\end{equation}
where $\theta$ are the parameters in the controller and the Lyapunov function, introduced at the beginning of this subsection \ref{subsec:trainer}. Unlike the traditional optimization problem, where the objective function is a closed-form expression of the decision variable $\theta$, in our problem \eqref{eq:bilevel_optimization} the objective function is the result of other maximization problems, whose coefficients and bounds of the constraint/cost matrices depend on $\theta$. In order to solve this min-max problem, we adopt an iterative procedure. In each iteration we first solve the inner maximization problem using MIP solvers, and then compute the gradient of the MIP optimal objective w.r.t the variables $\theta$, finally we apply gradient descent along this gradient direction, so as to reduce the objective in the outer minimization problem.

To compute the gradient of the maximization problem objective w.r.t $\theta$, after solving the inner MIP to optimality, we fix all the binary variables to their optimal solutions, and keep only the active linear constraints. The inner maximization problem can then be simplified to
\begin{subequations}
\begin{align}
	\gamma(\theta)=\max_{s} c_\theta^Ts + d_\theta\\
	\text{s.t } A_\theta s = b_\theta,
\end{align}
\label{eq:inner_maximization}%
\end{subequations}
where the problem coefficients/bounds $c_\theta, d_\theta, A_\theta, b_\theta$ are all explicit functions of $\theta$. $s$ contains all the continuous variables in the MIP, including $x_t$ and other slack variables. The optimal cost of \eqref{eq:inner_maximization} can be written in the closed form as $\gamma(\theta) = c_\theta^TA_\theta^{-1}b_\theta + d_\theta$, and then we can compute the gradient $\partial\gamma(\theta)/\partial\theta$ by back-propagating this closed-form expression. Note that this gradient is well defined if a tiny perturbation on $\theta$ changes only the optimal value of the continuous variables $s$, but not the set of active constraints or the optimal binary variable values (changing them would make the gradient ill-defined). This technique to differentiate the optimization objective w.r.t neural network parameters is becoming increasingly popular in the deep learning community. The interested readers can refer to \cite{amos2017optnet, agrawal2019differentiable} for a more complete treatment on differentiating an optimization layer.

Algorithm \ref{algorithm:bilevel} shows pseudo-code for this min-max optimization approach.
\begin{algorithm}
	\caption{Train controller/Lyapunov function through min-max optimization}
	\label{algorithm:bilevel}
	\begin{algorithmic}[1]
		\STATE Given a candidate control policy $\pi$ and a candidate Lyapunov function $V$.
		\WHILE{not converged}
		\STATE Solve MIP \eqref{eq:lyapunov_positivity_mip} and \eqref{eq:lyapunov_derivative_mip}.
			\IF{Either of MIP \eqref{eq:lyapunov_positivity_mip} of \eqref{eq:lyapunov_derivative_mip} has maximal objective $>0$}
			\STATE Compute the gradient of the MIP objectives w.r.t $\theta$, denote this gradient as $\partial \gamma/\partial\theta$.
			\STATE $\theta = \theta - \text{StepSize}*\partial\gamma/\partial\theta$.
			\ENDIF
		\ENDWHILE
	\end{algorithmic}
\end{algorithm}

\subsection{Computing region of attraction}
\label{subsec:roa}
After the training process in section \ref{subsec:trainer} converges to satisfy the Lyapunov condition for every state inside the bounded polytope $\mathcal{B}$, we compute an inner approximation of the  region of attraction for the closed-loop system. (Notice that the verified region $\mathcal{B}$ is \textit{not} a region of attraction, since it's not an invariant set, while the sub-level sets of $V$ are guaranteed to be invariant). One valid inner approximation is the largest sub-level set $\mathcal{S} = \{x_t | V(x_t)\le \rho\}$ contained inside the verified region $\mathcal{B}$, as illustrated in Fig. \ref{fig:roa}. Since we already obtained the Lyapunov function $V(x_t)$ in the previous section, we only need to find the largest value of $\rho$ such that $\mathcal{S}\subset \mathcal{B}$. 
Equivalently we can find $\rho$ through the following optimization problem
\begin{align}
	\rho = \min_{x_t\in\partial \mathcal{B}} V(x_t), \label{eq:roa_mip}
\end{align}
where the compact set $\partial \mathcal{B}$ is the boundary of the polytopic region $\mathcal{B}$, and the constraint $x_t\in\partial\mathcal{B}$ can be formulated as mixed-integer linear constraints (with one binary variable for a face of the polytope $\mathcal{B}$). As explained previously, the Lyapunov function $V(x_t)$ is a piecewise-affine function of $x_t$, hence the optimization problem \eqref{eq:roa_mip} is again an MIP, and can be solved efficiently by MIP solvers. 

It is worth noting that the size of this inner approximation of the region of attraction can be small, as we fix the Lyapunov function and only search for its sub-level set. To verify a larger inner approximation, one possible future research direction is to search for the Lyapunov function and the sub-level set simultaneously, as in \cite{richards2018lyapunov}.

\section{Results}
We synthesize stable controllers and Lyapunov functions on pendulum, 2D and 3D quadrotors. We use Gurobi as the MIP solver. All code runs on an Intel Xeon CPU. The sizes of the neural networks are shown in Table \ref{table:network_size} in Appendix \ref{subsec:network_structure}.

\subsection{Inverted pendulum}
\label{subsec:inverted_pendulum}
We first test our approach on an inverted pendulum. We approximate the pendulum Lagrangian dynamics using a neural network, by first simulating the system with many state/action pairs, and then approximating the simulation data through regression. To stabilize the pendulum at the top equilibrium $\theta=\pi, \dot{\theta}=0$, we synthesize a neural-network controller and a Lyapunov function using both Algorithm \ref{algorithm:train_on_sample} and \ref{algorithm:bilevel}. We verify the Lyapunov condition in the box region $0 \le \theta \le 2\pi, -5\le\dot{\theta}\le 5$. The Lyapunov function $V$ is shown in Fig. \ref{fig:pendulum_V}.
\begin{figure}
\begin{minipage}{0.21\textwidth}
\centering
	\includegraphics[width=0.98\textwidth]{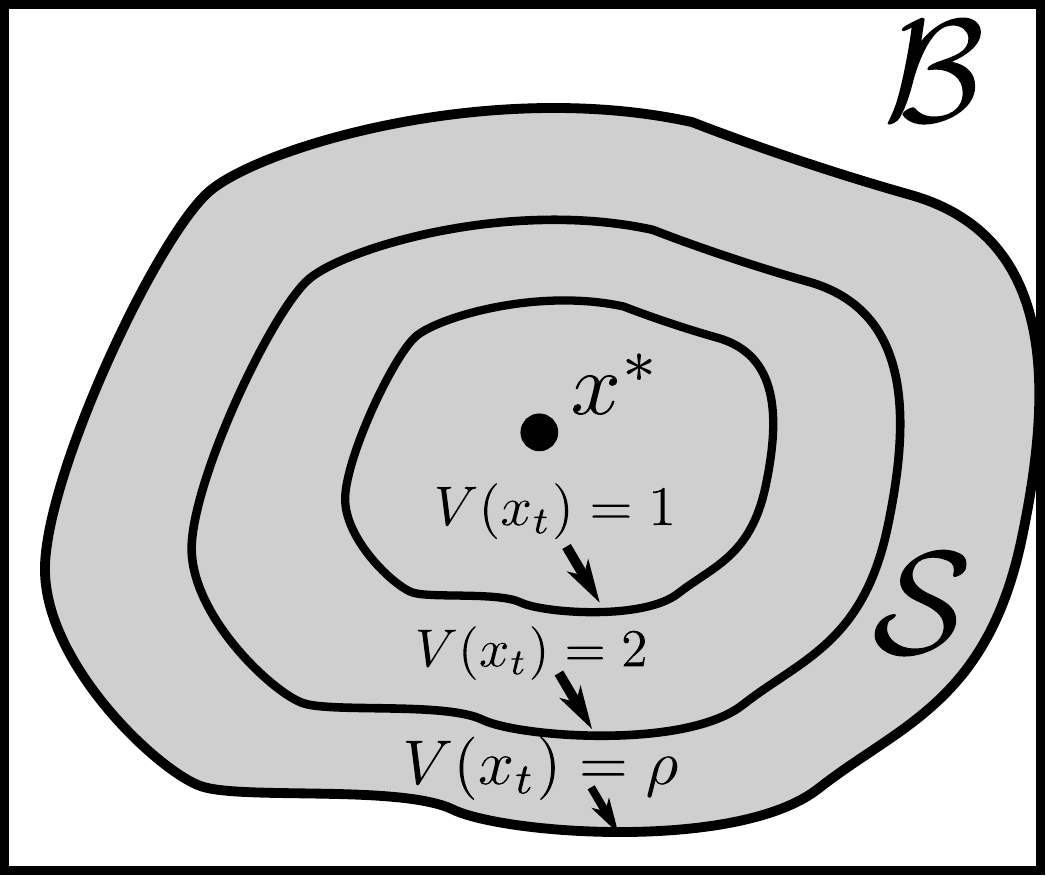}
	\caption{An inner approximation of the region of attraction $\mathcal{S}$ is the largest sub-level set $V(x_t)\le \rho$ contained inside the verified region $\mathcal{B}$, where the Lyapunov function is positive definite and strictly decreasing.}
	\label{fig:roa}
\end{minipage}
\hspace{2pt}
\begin{minipage}{0.26\textwidth}
	\centering
		\includegraphics[width=0.98\textwidth]{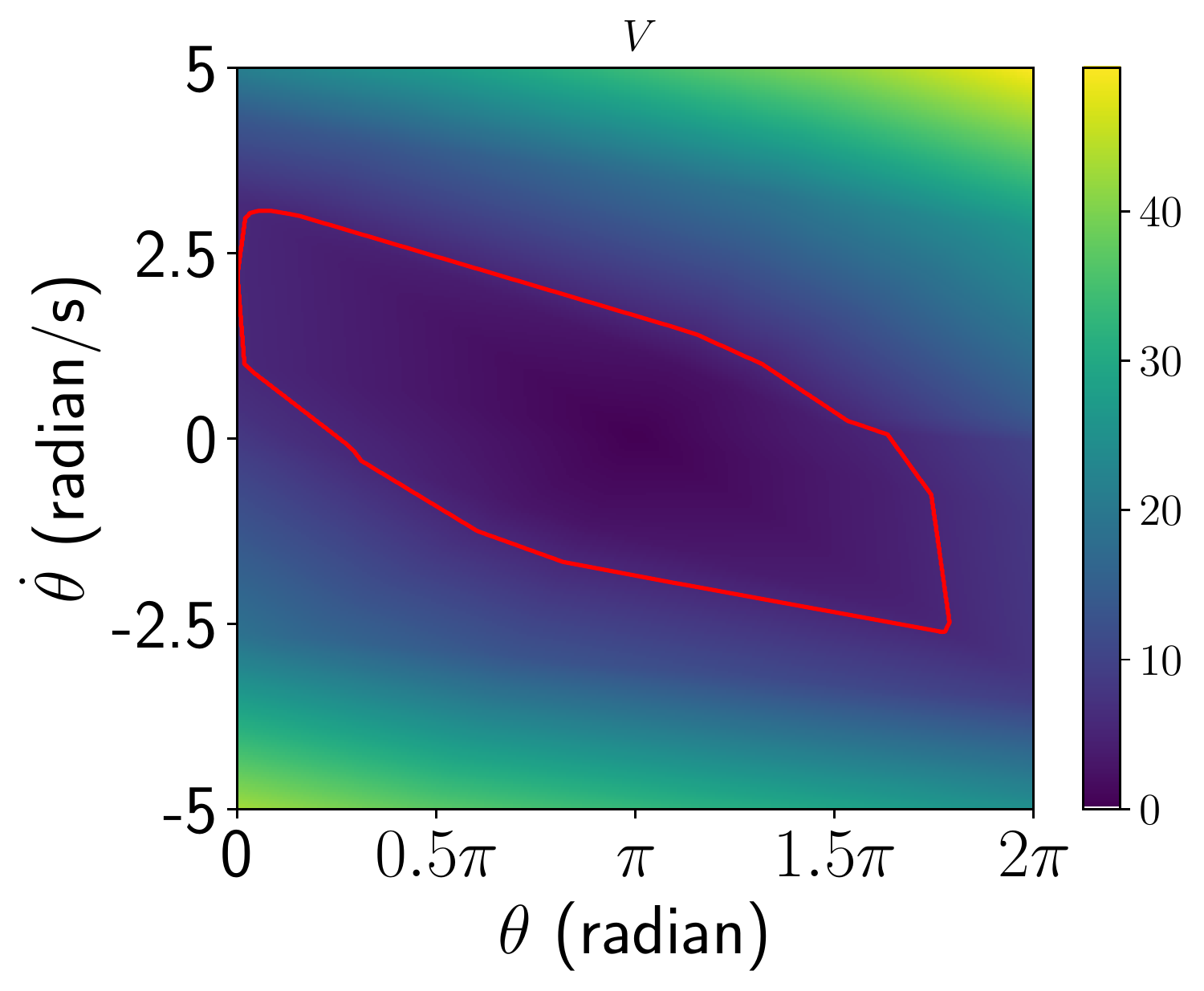}
	\caption{Heatmap of the Lyapunov function for the inverted pendulum. The red contour is the boundary of the verified inner approximation of the region of attraction, as the largest sub-level set contained in the verified box region $0\le \theta \le 2\pi, -5\le \dot{\theta}\le 5$.}
	\label{fig:pendulum_V}
	\end{minipage}
\end{figure}

We simulate the synthesized controller with the original pendulum Lagrangian dynamics model (not the neural network dynamics $\phi_{\text{dyn}}$). The result is shown in Fig. \ref{fig:pendulum_nn_simulate}. Although the neural network dynamics $\phi_{\text{dyn}}$ has approximation error, the simulation results show that the neural-network controller swings up and stabilizes the pendulum for not only the approximated neural network dynamics, but also for the original Lagrangian dynamics. Moreover, starting from many states outside the verified region of attraction, and even outside our verified box region, the trajectories still converge to the equilibrium. This suggests that the controller generalizes well. The small verified region of attraction suggests that in the future we can improve its size by searching over the Lyapunov function and the sub-level set simultaneously.

We start with a small box region $0.8 \pi \le \theta \le 1.2\pi, -1 \le \dot{\theta}\le 1$, and then gradually increase the verified region. We initialize the controller/Lyapunov network as the solution in the previous iteration on a smaller box region (at the first iteration, all parameters are initialized arbitrarily). For the smaller box $0.8\pi\le\theta\le1.2\pi, -1\le\dot{\theta}\le 1$, both algorihm \ref{algorithm:train_on_sample} and \ref{algorithm:bilevel} converge within a few minutes. For the larger box $0 \le\theta\le 2\pi, -5\le\dot{\theta}\le 5$, both algorithms converge within 3 hours.
\begin{figure}
	\begin{subfigure}{0.23\textwidth}
		\includegraphics[width=0.98\textwidth]{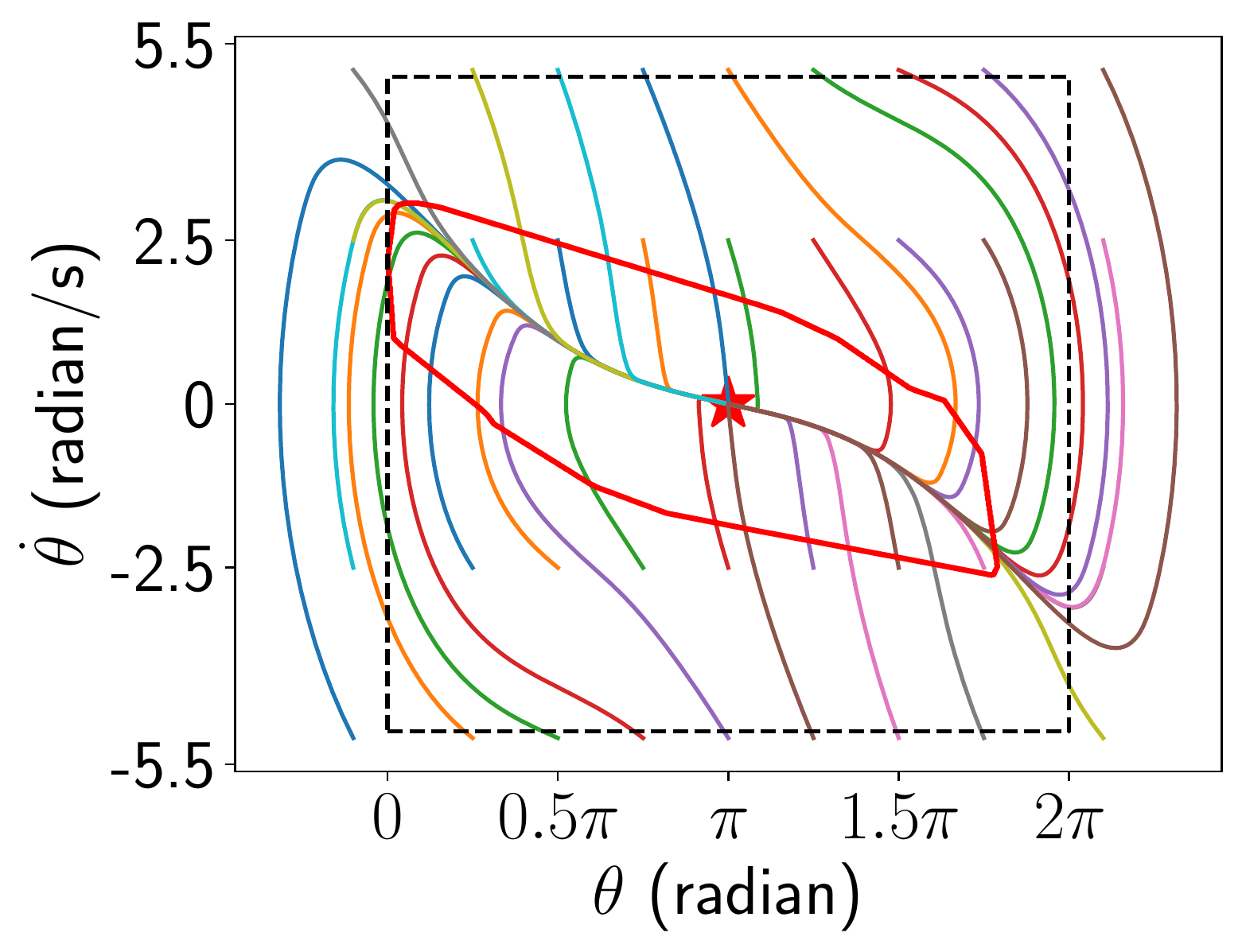}
	\end{subfigure}
	\begin{subfigure}{0.23\textwidth}
		\includegraphics[width=0.98\textwidth]{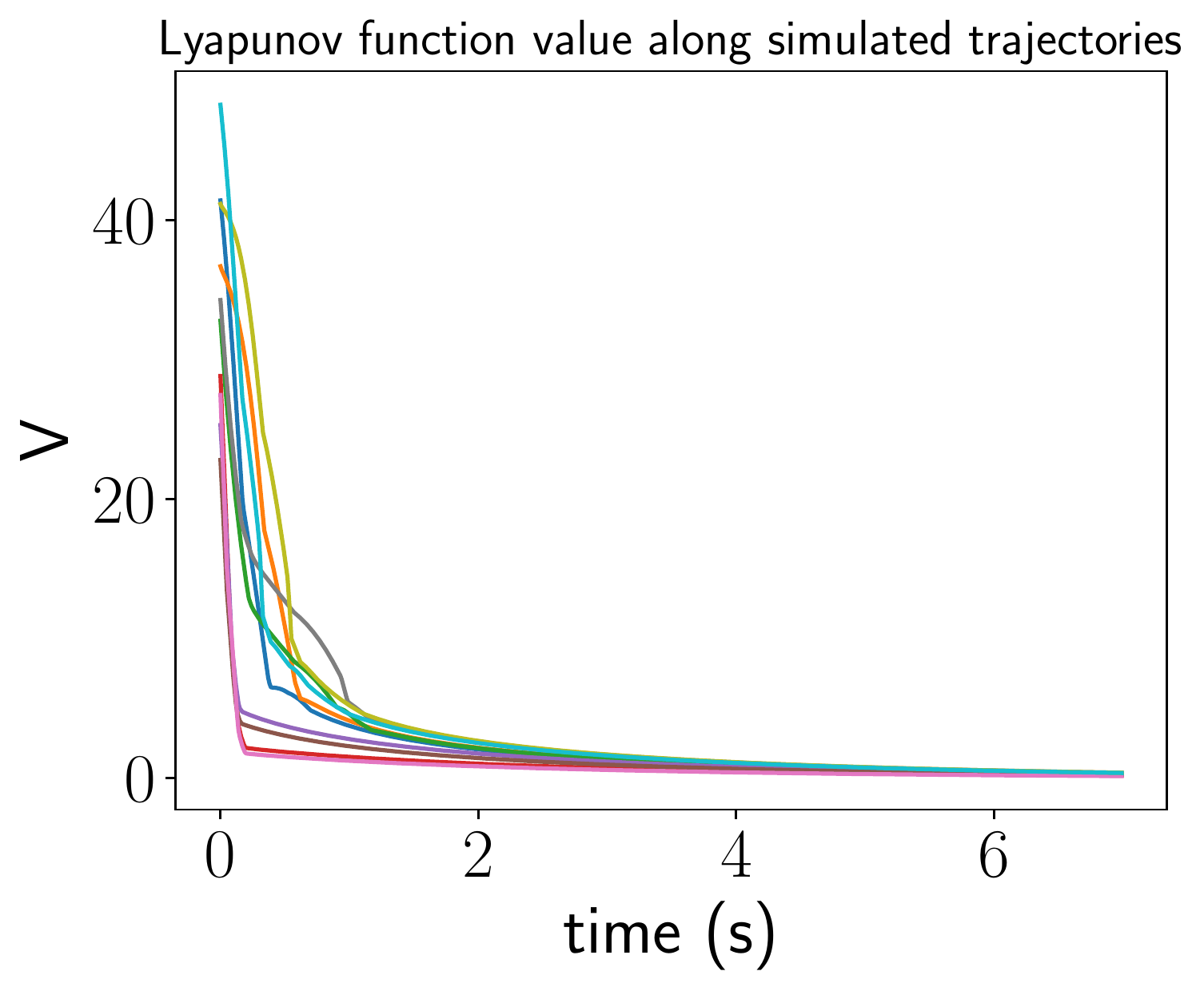}
	\end{subfigure}
	\caption{(left) phase plot of simulating the pendulum Lagrangian dynamics with the neural-network controller. The red contour is the boundary of the verified region of attraction, as the largest level set within the verified box region $\mathcal{B}$ (black dashed box). All the simulated trajectories (even starting outside of the dashed box) converge to the goal state. (right) Lyapunov function value along the simulated trajectories. The Lyapunov function decreases monotonically along the trajectories.}
	\label{fig:pendulum_nn_simulate}
	\vspace{-12pt}
\end{figure}

\subsection{2D quadrotor}
\label{subsec:quadrotor_2d}
We synthesize a stabilizing controller and a Lyapunov function for the 2D quadrotor model used in \cite{tedrake2009underactuated}. Again we first train a neural network $\phi_{\text{dyn}}$ to approximate the Lagrangian dynamics. Our goal is to steer the quadrotor to hover at the origin. In Fig.\ref{fig:quadrotor2d_snapshot} we visualize the snapshots of the quadrotor stabilized by our neural-network controller. We verified the Lyapunov conditions in the region $[-0.75, -0.75, -0.5\pi, -4, -4, -2.75] \le [x, z, \theta, \dot{x}, \dot{z}, \dot{\theta}] \le [0.75, 0.75, 0.5\pi, 4, 4, 2.75]$.
\begin{figure}
	\begin{subfigure}{0.24\textwidth}
		\includegraphics[width=0.97\textwidth]{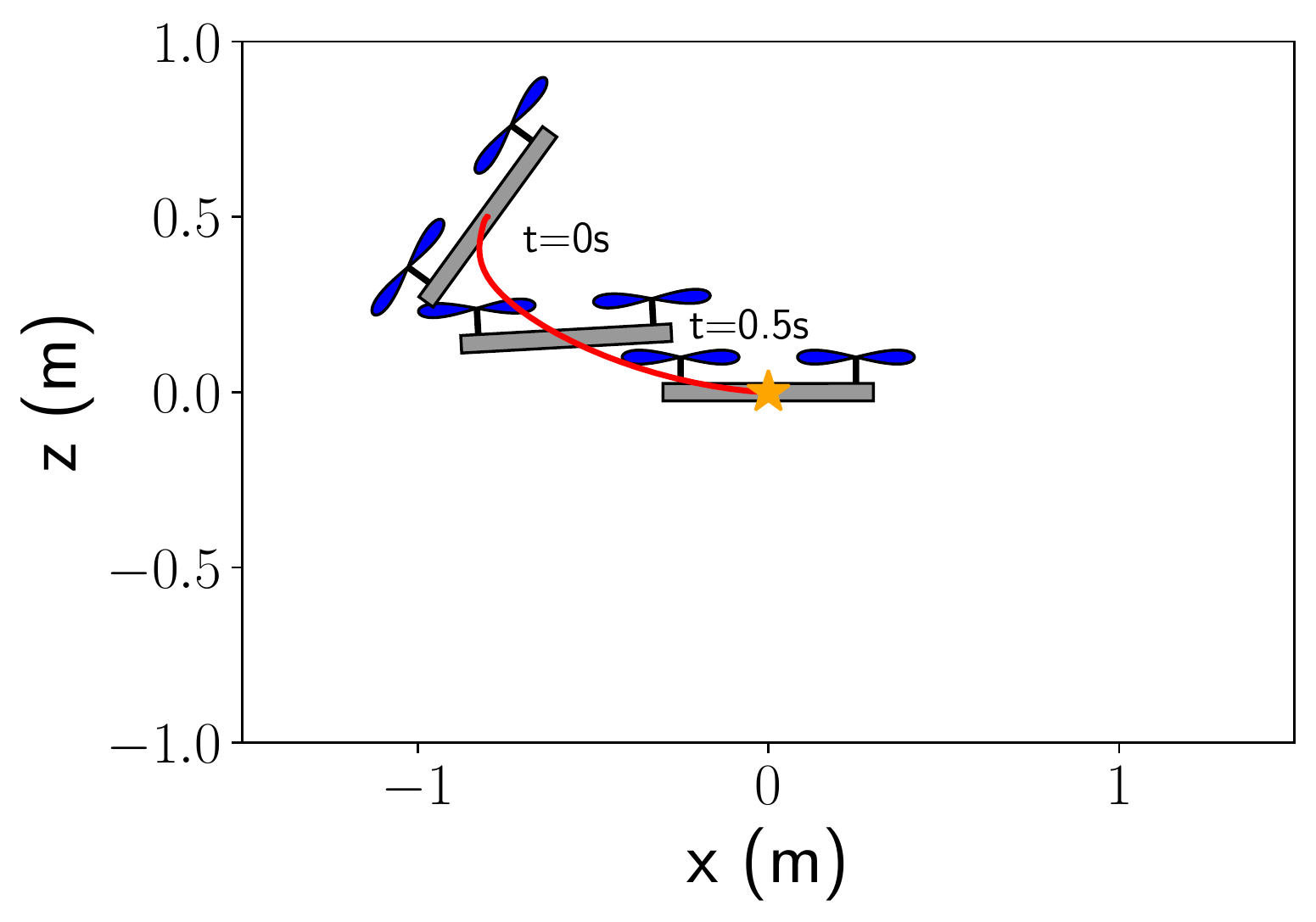}
	\end{subfigure}
	\begin{subfigure}{0.24\textwidth}
		\includegraphics[width=0.97\textwidth]{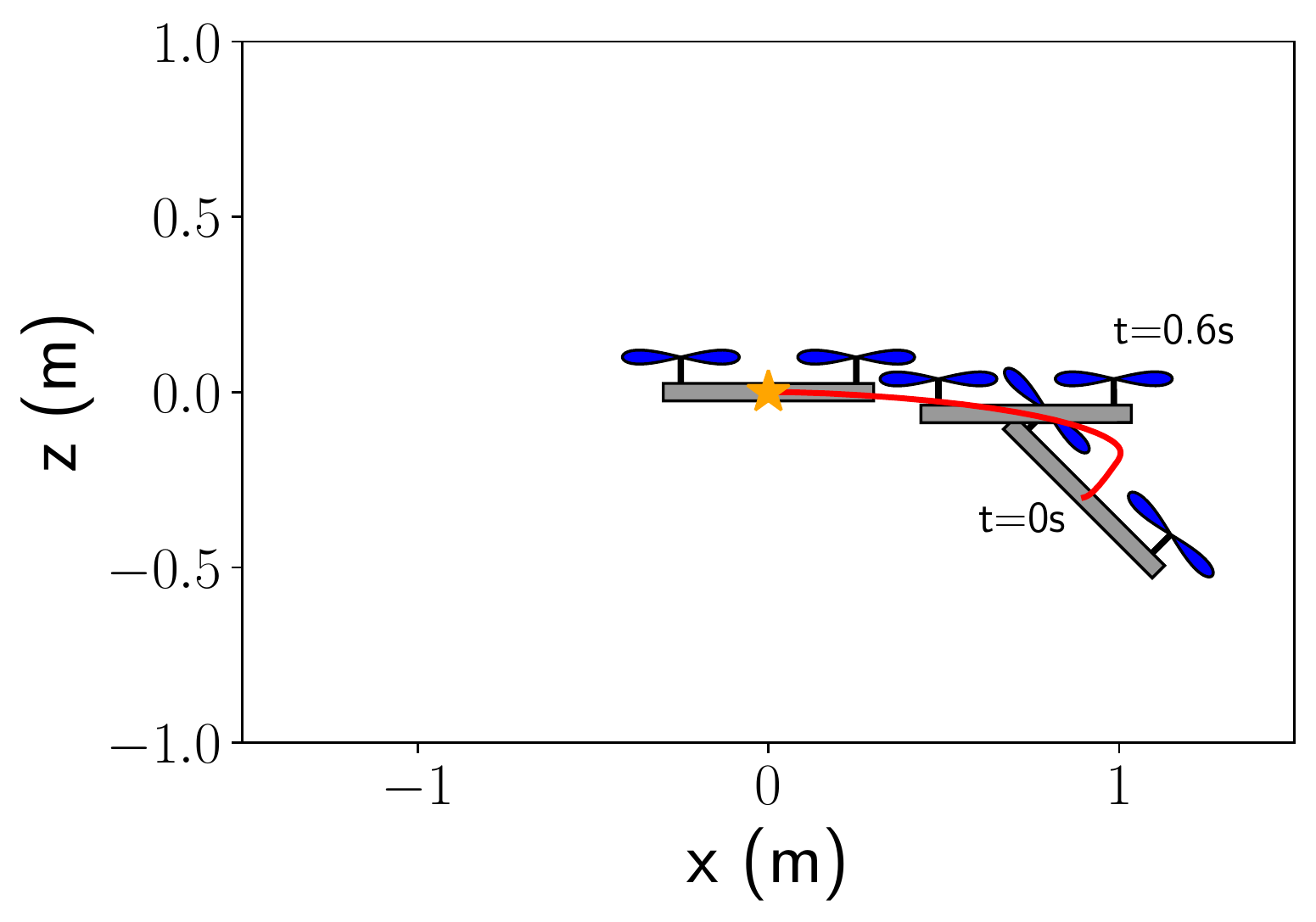}
	\end{subfigure}
	\caption{Snapshots of the 2D quadrotor simulation (with the original Lagrangian dynamics) using our neural-network controller from different initial states. The red lines are the trajectories of the quadrotor body frame origin.}
	\label{fig:quadrotor2d_snapshot}
\end{figure}

We sample 10000 initial states uniformly in the box $[-0.9, -0.9, -0.5\pi, -4.5, -4.5, -3] \le [x, z, \theta, \dot{x}, \dot{z},\dot{\theta}]\le[0.9, 0.9, 0.5\pi, 4.5, 4.5, 3]$. For each initial state we simulate the Lagrangian dynamics with the neural network and an LQR controller. We summarize the result in table \ref{table:quadrotor2d_sim} on whether the simulation converges to the goal state or not. More states can be stabilized by the neural-network controller than the LQR controller. Moreover, the off-diagonal entries in Table \ref{table:quadrotor2d_sim} demonstrates that the set of sampled states that are stabilized by the neural-network controller is a strict super-set of the set of states stabilized by the LQR controller. We believe there are two factors contributing to the advantage of our neural-network controller against an LQR: 1) the neural-network controller is piecewise linear while the LQR controller is linear; the latter can be a special case of the former. 2) the neural-network controller is aware of the input limits while the LQR controller is not.

\begin{table}
	\centering
	\begin{tabular}{|c | c | c|}
		\hline
		& NN succeeds & NN fails\\
		\hline
		LQR succeeds & 8078& 0\\
		\hline
		LQR fails& 1918& 4\\
		\hline
	\end{tabular}
	\caption{Number of success/failure for $10,000$ simulations of 2D quadrotor with the neural network (NN) controller and an LQR controller. The simulation uses the Lagrangian dynamics.}
	\label{table:quadrotor2d_sim}
\end{table}

We then focus on certain two dimensional slices of the state space, and sample many initial states on these slices. For each sampled initial state we simulate the Lagrangian dynamics using both the neural-network and the LQR controller. We visualize the simulation results in Fig. \ref{fig:quadrotor2d_lqr_vs_nn}. Each dot represents a sampled initial state, and we color each initial state based on whether the neural-network (NN)/LQR controllers succeed in stabilizing that initial state to the goal
\begin{itemize}
    \item Purple: NN succeeds but LQR fails.
    \item Green: both NN and LQR succeed
    \item Red: both NN and LQR fail.
\end{itemize}
Evidently the large purple region suggests that the region of attraction with the neural-network controller is a strict super-set of that with the LQR controller. We observe that for the initial states where LQR fails, the LQR controller requires thrusts beyond the input limits. If we increase the input limits then LQR can stabilize many of these states. Hence by taking input limits into consideration, the neural-network controller achieves better performance than the LQR.
\begin{figure}
	\begin{subfigure}{0.23\textwidth}
		\includegraphics[width=0.98\textwidth]{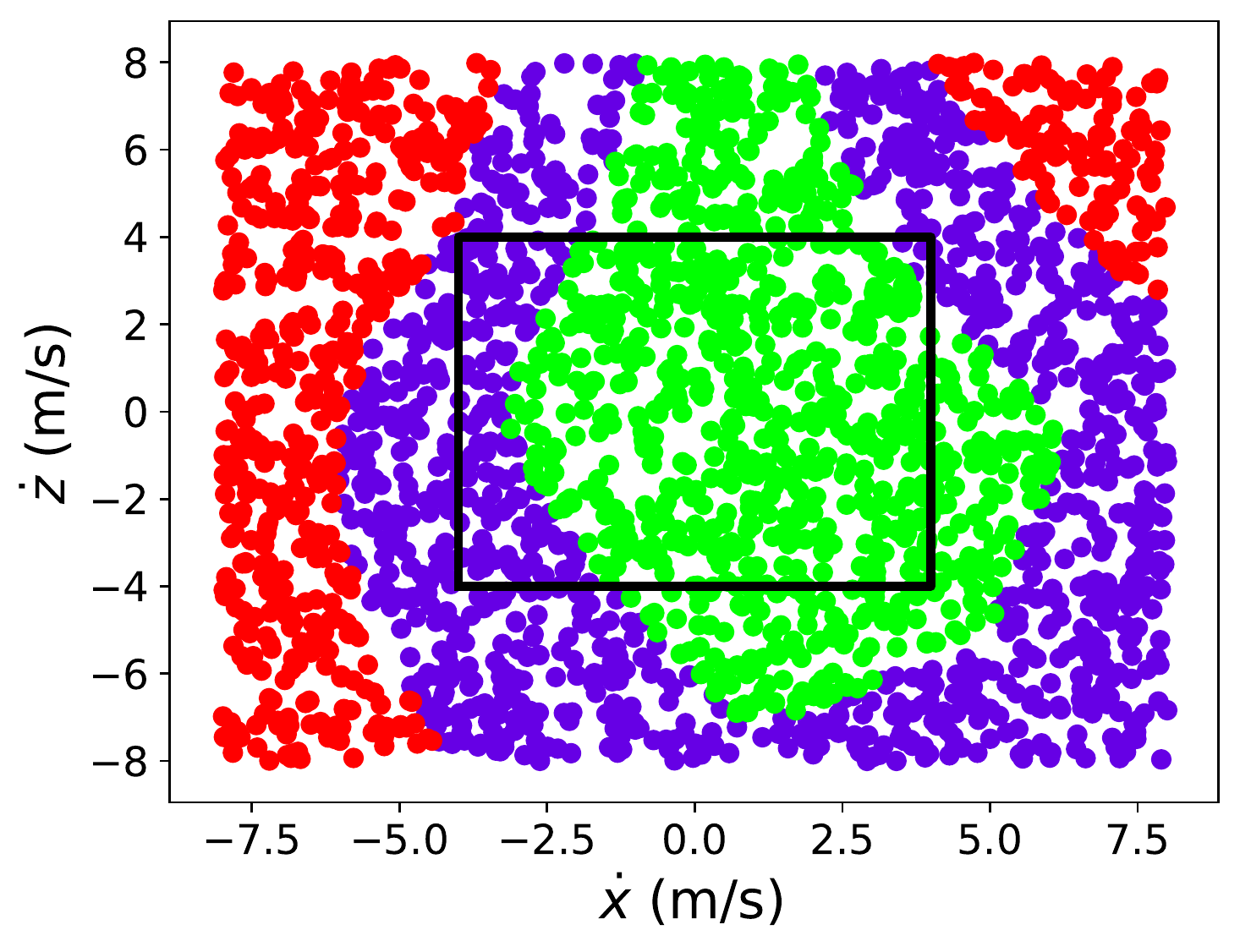}
	\end{subfigure}
	\begin{subfigure}{0.23\textwidth}
		\includegraphics[width=0.98\textwidth]{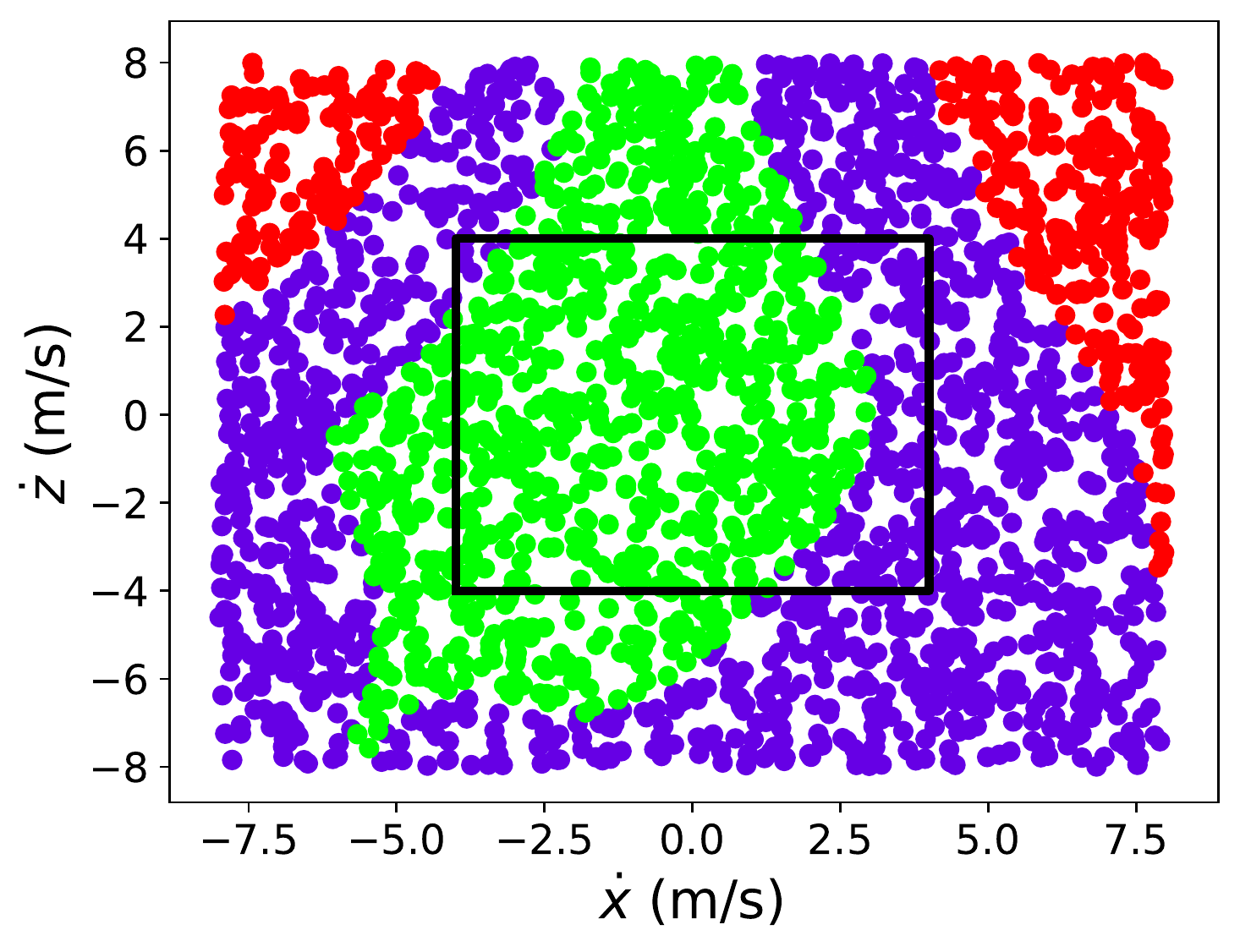}
	\end{subfigure}
	\caption{We sample 2500 initial states within the box region $[-8, -8]\le [\dot{x}, \dot{z}] \le [8, 8]$, with $[x, z, \theta,  \dot{\theta}]$ fixed to $[-0.75, 0.3, 0.3\pi, 2]$ (left), and $[0.75, 0.5, -0.4\pi, 2]$ (right). We simulate from each initial state with the neural-network controller (NN) and the LQR controller, and color each initial state based on whether the simulation converges to the goal. All red dots (where the NN controller fails) are outside of the black box region within which we verified the Lyapunov conditions.} 
	\label{fig:quadrotor2d_lqr_vs_nn}
\end{figure}

Both algorithm \ref{algorithm:train_on_sample} and \ref{algorithm:bilevel} find the stabilizing controller. For a small box region $[-0.1, -0.1, -0.1\pi, -0.5, -0.5, -0.3] \le [x, z, \theta, \dot{x}, \dot{z}, \dot{\theta}]\le [0.1, 0.1, 0.1\pi, 0.5, 0.5, 0.3]$, both algorithms converge in 20 minutes. For the larger box used in Table \ref{table:quadrotor2d_sim}, the algorithms converge in 1 day.

\subsection{3D quadrotor}
We apply our approach to a 3D quadrotor model with 12 states \cite{mellinger2011minimum}. Again, our goal is to steer the quadrotor to hover at the origin. As visualized in Fig.\ref{fig:quadrotor3d_snapshot}, our neural-network controller can stabilize the system. Training this controller took 3 days.
\begin{figure}
	\begin{subfigure}{0.24\textwidth}
		\includegraphics[width=0.98\textwidth]{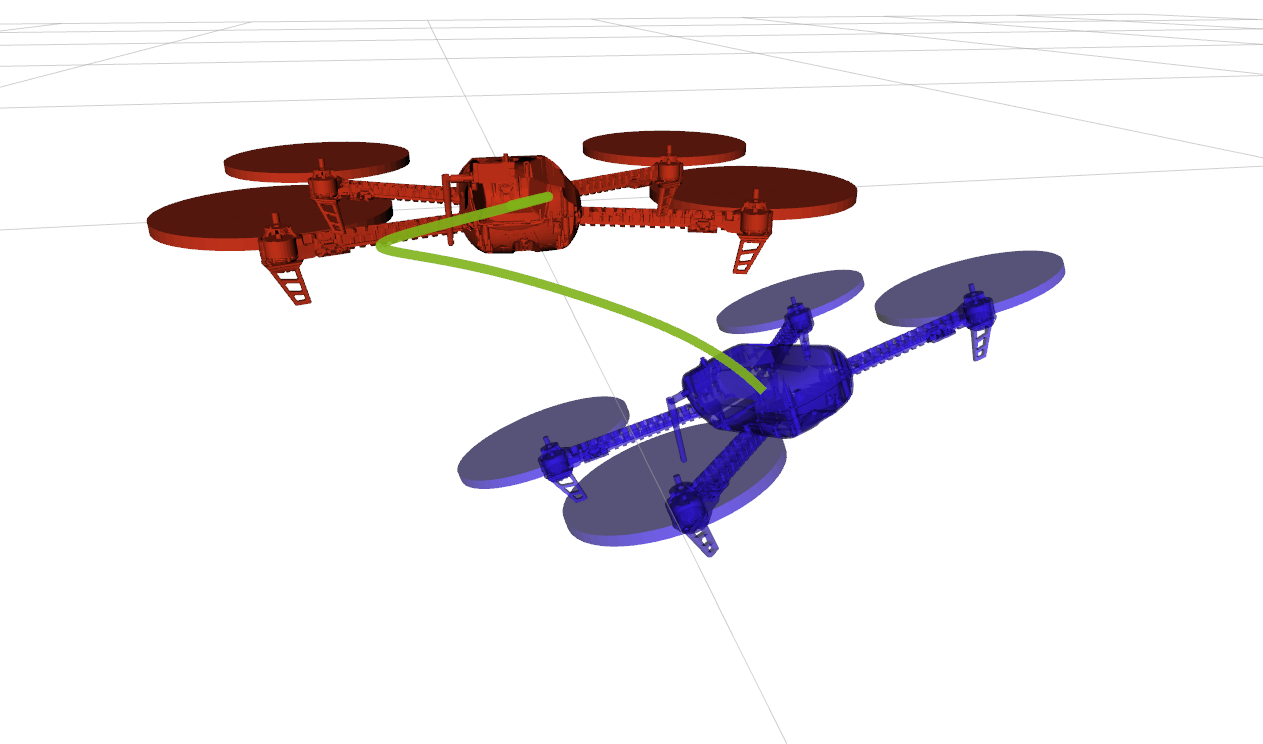}
	\end{subfigure}
	\begin{subfigure}{0.24\textwidth}
		\includegraphics[width=0.98\textwidth]{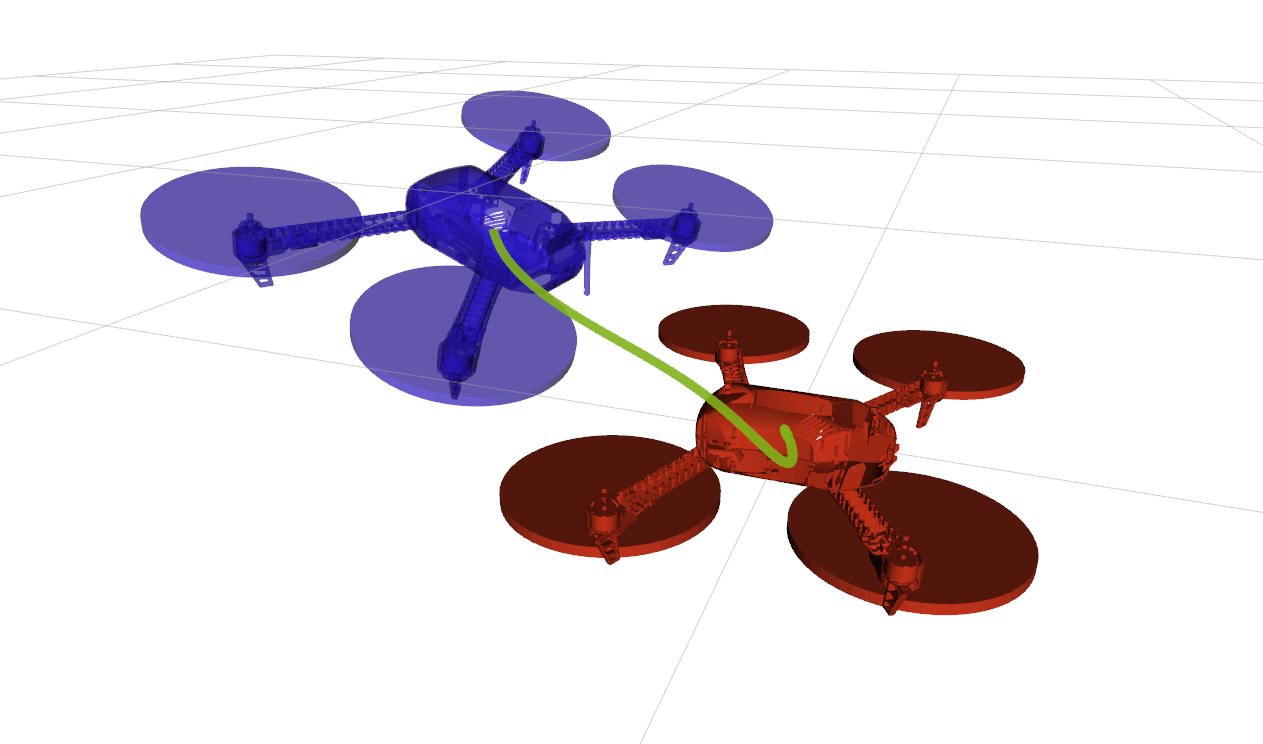}
	\end{subfigure}
	\caption{Snapshots of simulating the quadrotor using our neural-network controller with the Lagrangian dynamics. The quadrotor converge to the hovering state at the origin (red).}
	\label{fig:quadrotor3d_snapshot}
\end{figure}

The quadrotor Lagrangian dynamics is approximated by a neural network $\phi_{\text{dyn}}$ with comparatively large mean-squared error (MSE) around $10^{-4}$ (reducing MSE would require a neural network too large for our MIP solver), while other examples in this paper have MSE in the order of $10^{-6}$. Hence there are noticeable discrepancies between the simulation with Lagrangian dynamics and with the neural network dynamics $\phi_{\text{dyn}}$. In Fig \ref{fig:quadrotor3d_V_traj} we select some results to highlight the discrepancy, that the Lyapunov function always decreases along the trajectories simulated with neural-network dynamics, while it could increase with Lagrangian dynamics. Nevertheless, the quadrotor eventually always converges to the goal state. We note that the same phenomenon would also happen if we took a linear approximation of the quadrotor dynamics and stabilized the quadrotor with an LQR controller. If we were to plot the quadratic Lyapunov function (which is valid for the LQR controller and the linearized dynamics), that Lyapunov function could also increase along the trajectories simulated with the nonlinear Lagrangian dynamics (see Fig \ref{fig:quadrotor3d_lqr_V_traj} in the Appendix). Analogous to approximating the nonlinear dynamics with a linear one and stabilizing it with a linear LQR controller, our approach can be regarded as approximating the nonlinear dynamics with a neural network and stabilizing it with another neural-network controller. 
\begin{figure}
	\begin{subfigure}{0.24\textwidth}
		\includegraphics[width=0.98\textwidth]{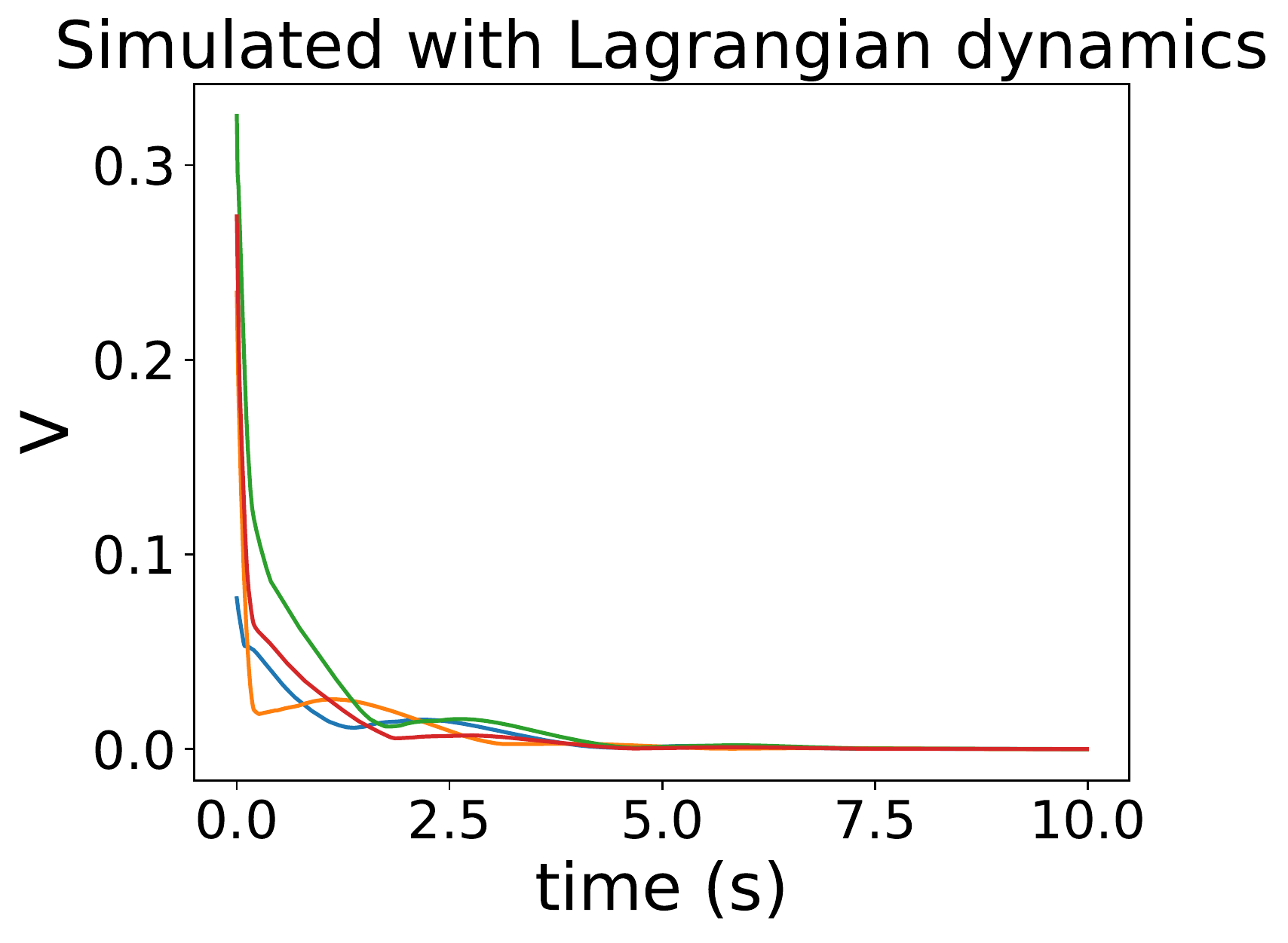}
	\end{subfigure}
	\begin{subfigure}{0.24\textwidth}
		\includegraphics[width=0.98\textwidth]{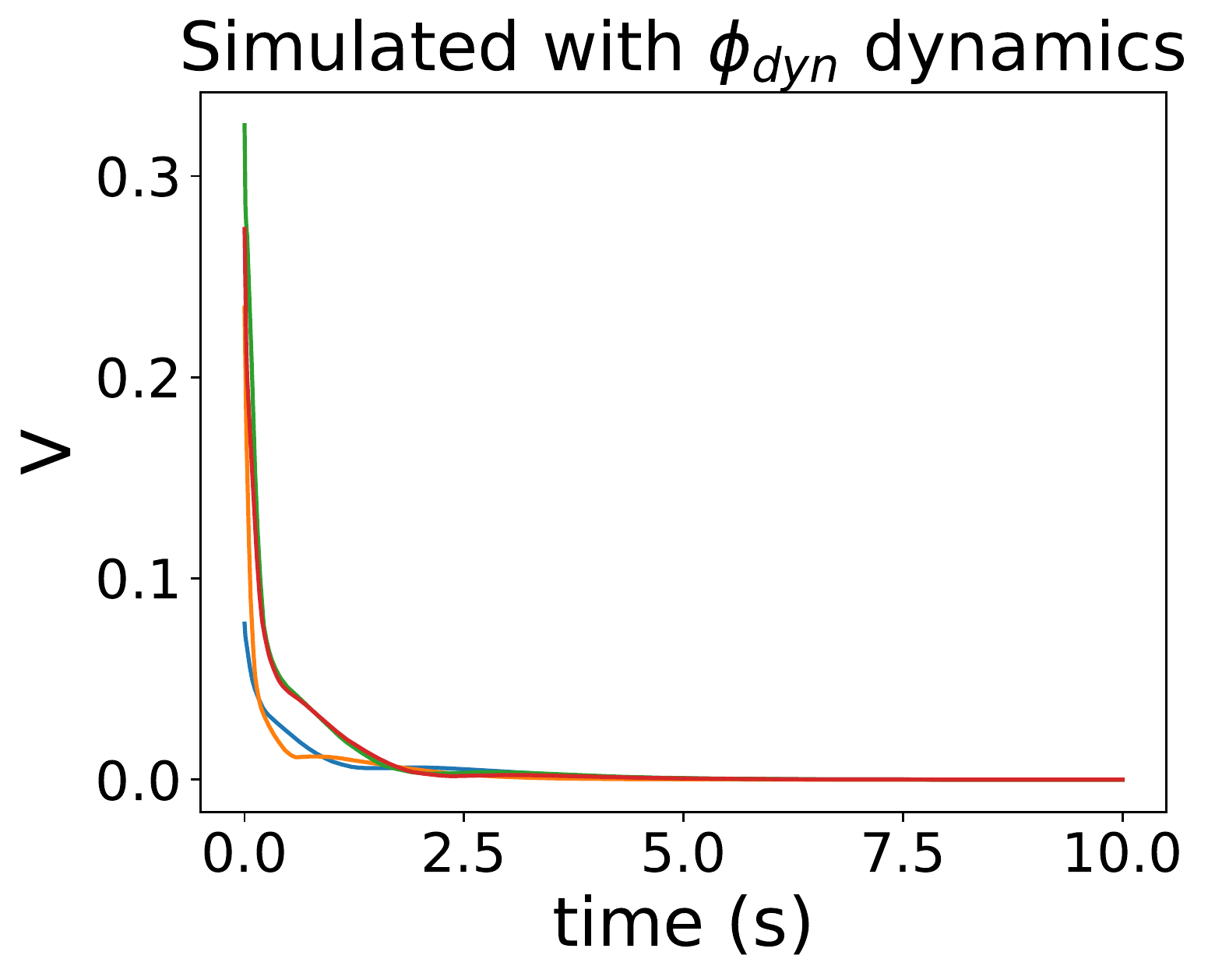}
	\end{subfigure}
	\caption{3D quadrotor Lyapunov function value along the simulated trajectories with our neural-network controller. The quadrotor is simulated with Lagrangian dynamics (left) vs the dynamics approximated by a neural network $\phi_{\text{dyn}}$ (right). In both left and right sub-plots, the initial states are the same.}
	\vspace{-10pt}
	\label{fig:quadrotor3d_V_traj}
\end{figure}

Finally we compare the performance of Algorithm \ref{algorithm:train_on_sample} which appends counter examples to training sets, against Algorithm \ref{algorithm:bilevel} with min-max optimization. We take 10 runs of each algorithm, and report the average computation time for each algorithm in Table \ref{table:algorithm_time} \footnote{There are 2 failed runs with Algorithm \ref{algorithm:train_on_sample} on the 2D quadrotor example, that they time out after 6 hours, and are not included in Table \ref{table:algorithm_time}. For the 3D quadrotor we only include 3 runs as they are too time-consuming}. For the small-sized task (pendulum with 2 states), Algorithm \ref{algorithm:train_on_sample} is orders of magnitude faster than Algorithm \ref{algorithm:bilevel}, and they take about the same time on the medium-sized task (2D quadrotor with 6 states). On the large-sized task (3D quadrotor with 12 states), Algorithm \ref{algorithm:train_on_sample} doesn't converge while Algorithm \ref{algorithm:bilevel} can find the solution. We speculate this is because Algorithm \ref{algorithm:train_on_sample} overfits to the training set. For a small-sized task the overfitting is not a severe problem as a small number of sampled states are sufficient to represent the state space; while for a large-sized task it would require a huge number of samples to cover the state space. With the limited number of counter examples Algorithm \ref{algorithm:train_on_sample} overfits to these samples while causing large Lyapunov condition violation elsewhere. This is evident from the loss curve plot in Fig.\ref{fig:loss_for_algorithms} for a 2D quadrotor task. Although both algorithms converge, the loss curve decreases steadily with Algorithm \ref{algorithm:bilevel}, while it fluctuates wildly with Algorithm \ref{algorithm:train_on_sample}. We believe that the fluctuation is caused by overfitting to the training set in the previous iteration. Nevertheless, this comparison is not yet conclusive, and we are working to improve the performance of Algorithm \ref{algorithm:train_on_sample} on the large-size task.

\begin{table}
\centering
\begin{tabular}{|c|c|c|}
\hline
      &  Algorithm \ref{algorithm:train_on_sample} &Algorithm \ref{algorithm:bilevel}\\
     \hline
     Pendulum&8.4s& 224s\\
     2D quadrotor &948.3s & 1004.7s\\
     3D quadrotor &Time out after 5 days&65.7hrs\\
     \hline
\end{tabular}
\caption{Average computation time of 10 runs for both algorithms. To speed up the computation, the verified region $\mathcal{B}$ is relatively small.}
\label{table:algorithm_time}
\end{table}
\begin{figure}
    \centering
    \begin{subfigure}{0.24\textwidth}
    \includegraphics[width=0.98\textwidth]{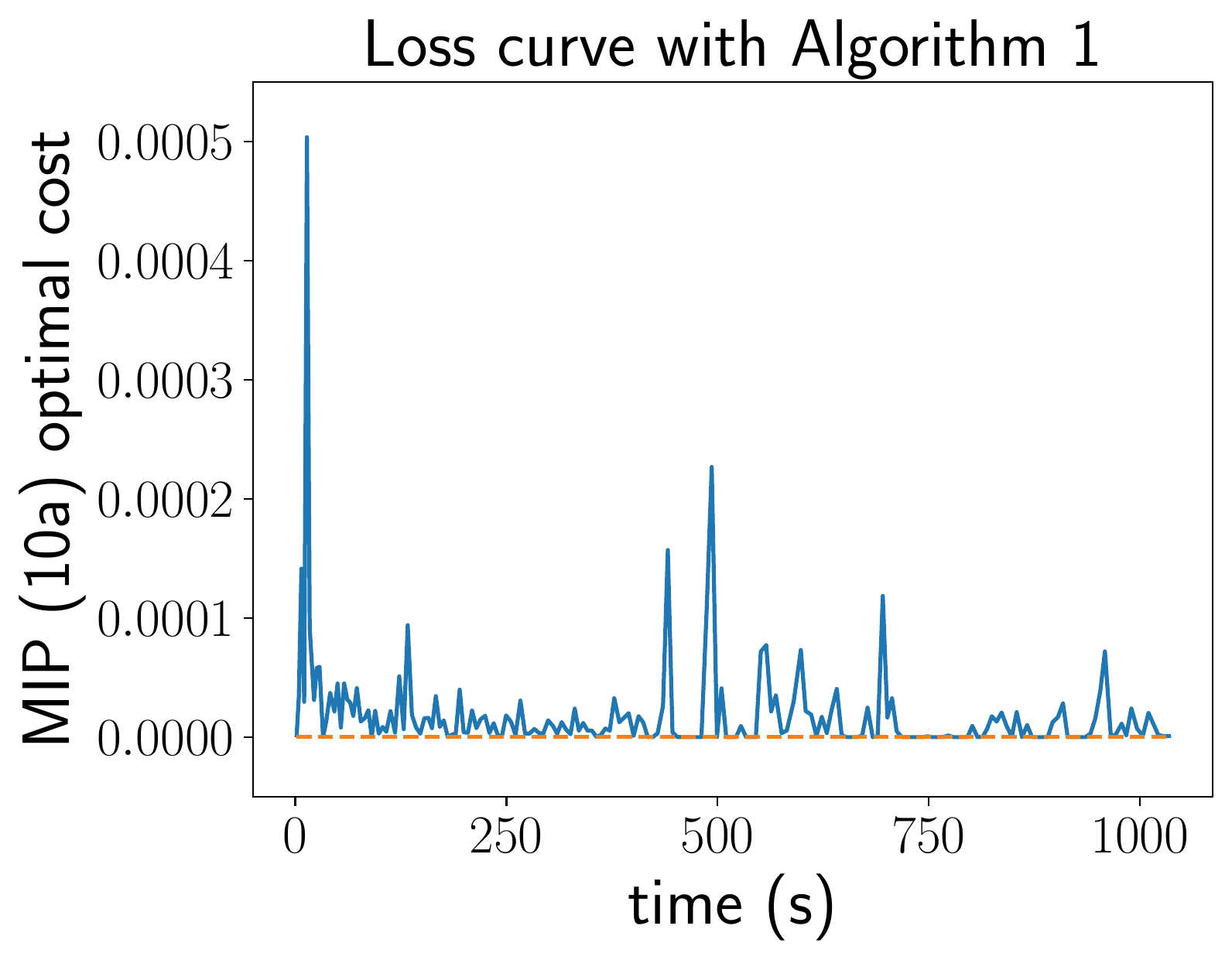}
    \end{subfigure}
    \begin{subfigure}{0.24\textwidth}
    \includegraphics[width=0.98\textwidth]{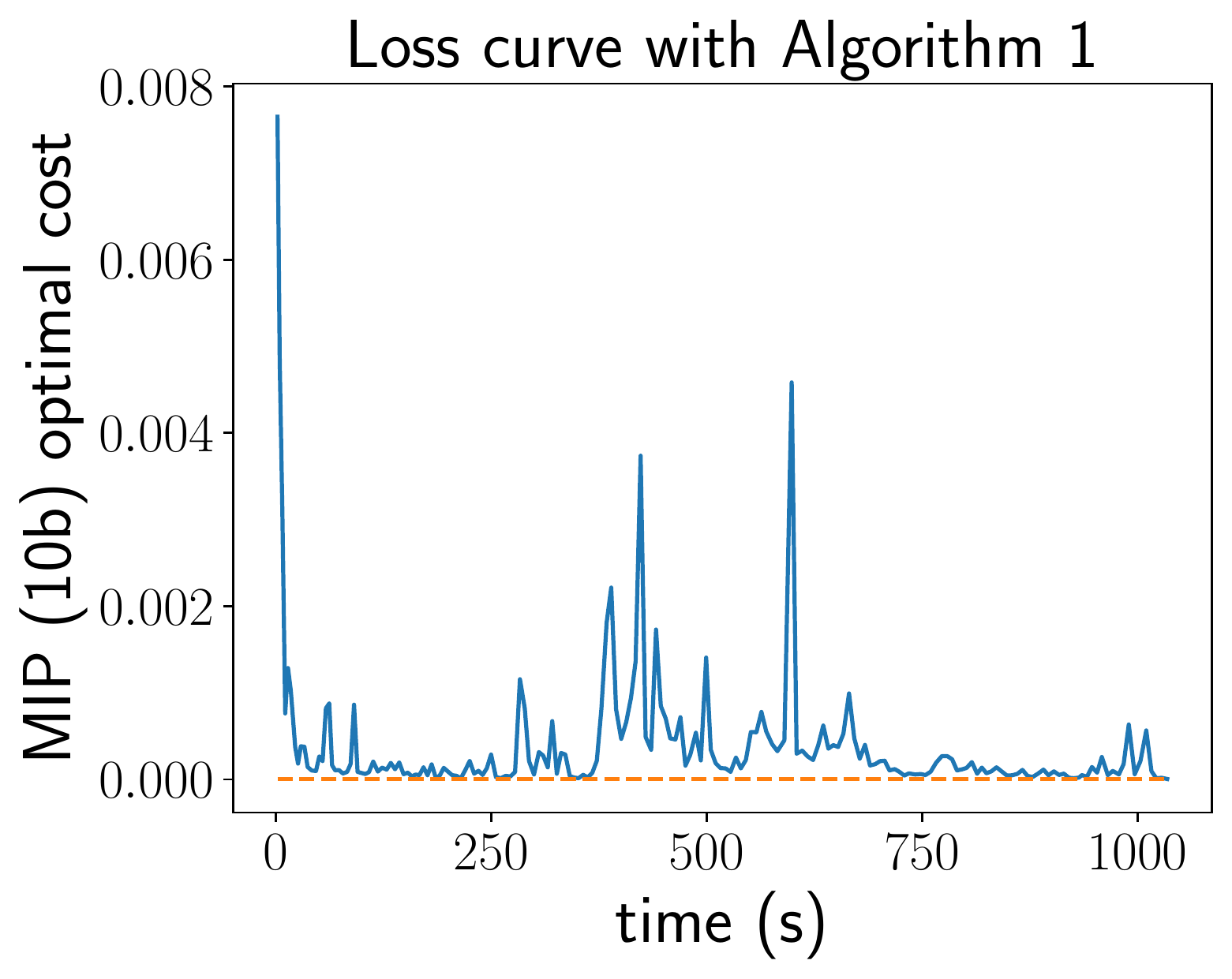}
    \end{subfigure}
    \begin{subfigure}{0.24\textwidth}
    \includegraphics[width=0.98\textwidth]{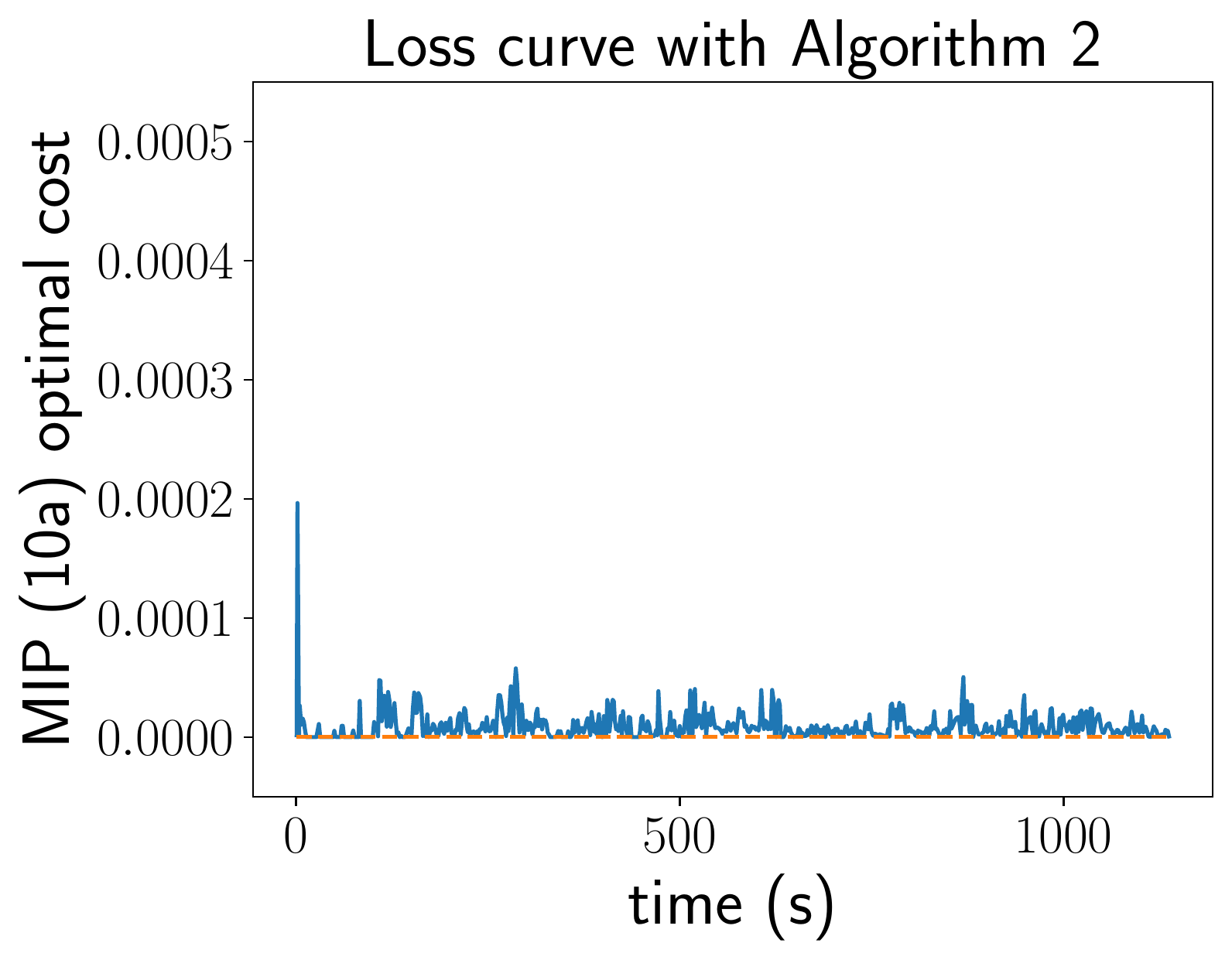}
    \end{subfigure}
    \begin{subfigure}{0.24\textwidth}
    \includegraphics[width=0.98\textwidth]{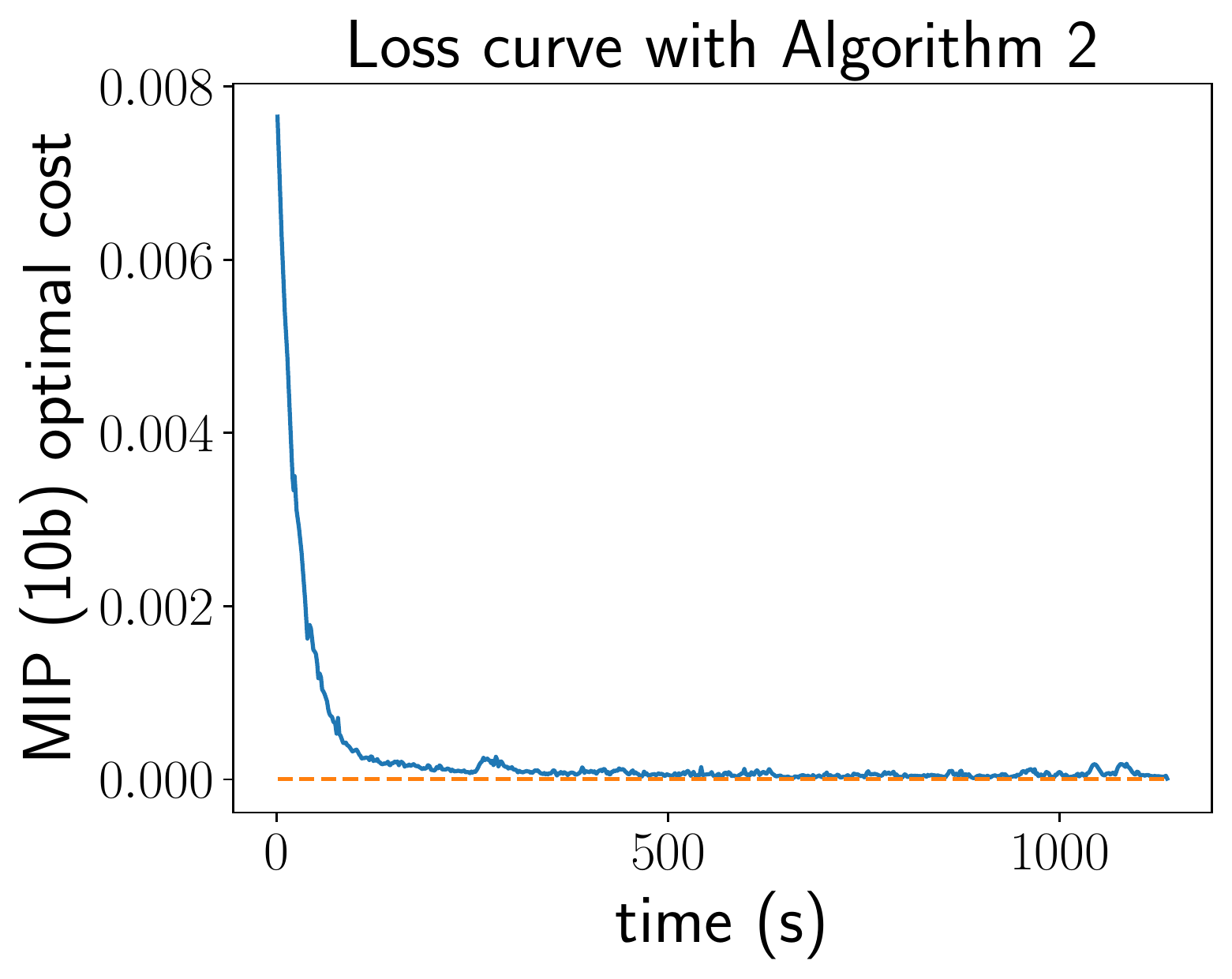}
    \end{subfigure}
    \caption{Loss curves on the 2D quadrotor task for Algorithm \ref{algorithm:train_on_sample} and \ref{algorithm:bilevel}}
    \label{fig:loss_for_algorithms}
    \vspace{-10pt}
\end{figure}
\section{Conclusion and Future work}
In this paper, we demonstrate a method to synthesize a neural-network controller to stabilize a robot system, as well as a neural-network Lyapunov function to prove the resulting stability. We propose an MIP verifier to either certify Lyapunov stability, or generate counter examples that can be used to improve the candidate controller and the Lyapunov function. We present another MIP to compute an inner approximation of the region of attraction. We demonstrate our approach on an inverted pendulum, 2D and 3D quadrotors, and showcase that it can outperform a baseline LQR controller.

Currently, the biggest challenge of our approach is scalability. The speed bottleneck lies in solving MIPs, where the number of binary variables scales linearly with the total number of neurons in the networks. In the worst case, the complexity of solving an MIP scales exponentially with the number of binary variables, when the solver has to check every node of a binary tree. However, in practice, the branch-and-cut process significantly reduces the number of nodes to explore. Recently, with the growing interest from the machine learning community, many approaches were proposed to speed up verifying neural networks through MIP by tightening the formulation \cite{anderson2020strong, tsay2021partition}. We plan to explore these approaches in the future.


Our proposed MIP formulation works for discrete-time dynamical systems. For continuous-time dynamical systems, neural networks have been previously used either to approximate the system dynamics \cite{lutter2018deep}, or to synthesize optimal controllers \cite{farshidian2019deep}. We plan to extend our approach to continuous-time dynamical systems. Moreover, we can readily apply our approach to systems whose dynamics are approximated by piecewise-affine dynamical systems, such as soft robots \cite{tonkens2020soft} and hybrid systems with contact \cite{marcucci2019mixed}, since piecewise-affine dynamic constraints can easily be encoded into MIP.

Many safety-critical missions also require the robot to avoid unsafe regions. We can readily extend our framework to synthesize barrier functions \cite{ames2016control} so that the robot certifiably stays within the safe region.

\section*{Acknowledgments}
Benoit Landry is sponsored by the NASA University Leadership initiative (grant $\#$80NSSC20M0163), and Lujie Yang is sponsored by Amazon Research Award $\#$6943503. This article solely reflects the opinions and conclusions of its authors and not any funding agencies. We would like to thank Vincent Tjeng and Shen Shen for the valuable discussion.

\bibliographystyle{plainnat}
\bibliography{references}

\section{Appendix}
\subsection{Network structures in each example}
\label{subsec:network_structure}
For each task in the result section, we use three fully connected feedforward neural networks to represent forward dynamics, the control policy and the Lyapunov function respectively. All the networks use leaky ReLU activation function. The size of the network is summarized in Table \ref{table:network_size}. Each entry represents the size of the hidden layers. For example $(3, 4)$ represents a neural network with 2 hidden layers, the first hidden layer has 3 neurons, and the second hidden layer has 4 neurons.
\begin{table}
\centering
	\begin{tabular}{|m{1.7cm}|m{1.5cm}|m{1.5cm}|m{1.5cm}|}
		\hline
		& forward dynamics network $\phi_{\text{dyn}}$ & controller network $\phi_\pi$ & Lyapunov network $\phi_V$\\
		\hline
		Inverted pendulum & (5, 5) & (2, 2) & (8, 4, 4)\\
		\hline
		2D quadrotor & (7, 7) & (6, 4)& (10, 10, 4)\\
		\hline
		3D quadrotor & (10, 10) & (16, 8) &(16, 12, 8)\\
		\hline
	\end{tabular}
	\caption{Size of the neural network hidden layers in each task.}
	\label{table:network_size}
\end{table}

\subsection{MIP formulation for $l_1$ norm and clamp function}
\label{subsec:l1_clamp_MIP}
For the $l_1$ norm constraint on $x\in\mathbb{R}^n$
\begin{align}
	y=|x|_1\label{eq:l1}
\end{align}
where $x$ is bounded elementwisely as $l \le x \le u$, we convert this constraint \eqref{eq:l1} to the following mixed-integer linear constraint
\begin{subequations}
\begin{align}
	y = z_1+\hdots+z_n\\
	z_i\ge x_i,\;
	z_i\ge-x_i\\
	z_i \le x_i + 2l_i(\alpha_i-1),\;
	z_i \le 2u_i\alpha_i - x_i\\
	\alpha \in \{0, 1\}
\end{align}
\end{subequations}
where we assume $l< 0, u > 0$ (the case when $l \ge 0$ or $u\le 0$ is trivial).

For a clamp function
\begin{align}
	y = \begin{cases} l,\text{ if } x\le l\\ x,\text{ if } l\le x \le u\\ u , \text{ if } x\ge u\end{cases} \label{eq:clamp}
\end{align}
This clamp function can be rewritten in the following form using ReLU function
\begin{align}
	y =u - ReLU(u - (ReLU(x-l) + l))
\end{align}
As explained in \eqref{eq:leaky_relu_mix_integer} in section \ref{sec:background}, we can convert ReLU function to mixed-integer linear constraints, hence we obtain the MIP formulation of \eqref{eq:clamp}.

\subsection{Necessity of condition \eqref{eq:lyapunov_positivity_l1}}
\label{subsec:l1_necessary}
\begin{lemma}
There exists a piecewise-affine function $V(x):\mathbb{R}^n\rightarrow\mathbb{R}$ satisfying 
\begin{align}
    V(x^*) = 0, V(x) > 0\; \forall x\neq x^*\label{eq:V_positive}
\end{align}
if and only if for a given positive scalar $\epsilon>0$ and a given full column-rank matrix $R$, there exists another piecewise-affine function $\bar{V}(x):\mathbb{R}^n\rightarrow \mathbb{R}$ satisfying
\begin{align}
    \bar{V}(x^*)=0, \bar{V}(x) \ge \epsilon|R(x-x^*)|_1\;\forall x\neq x^*\label{eq:V_ge_l1}
\end{align}
\end{lemma}
\begin{proof}
The ``only if" part is trivial, if such $\bar{V}(x)$ exists, then just setting $V(x) = \bar{V}(x)$ and \eqref{eq:V_positive} holds. To prove the ``if" part, assume that $V(x)$ exists, and we will show that there exists a positive scalar, such that by scaling $V(\cdot)$ we get $\bar{V}(\cdot)$. Intuitively this scalar could be found as the smallest ratio between $V(x)$ and $\epsilon|R(x-x^*)|_1$. Formally, given a unit length vector $d\in\mathbb{R}^n$, we define a scalar function $\phi_d(t) = V(x^*+td)$, this scalar function $\phi(\cdot):\mathbb{R}\rightarrow\mathbb{R}$ is also piecewise-affine, as it is just the value of $V(\cdot)$ along the direction $d$. Likewise we define another scalar function $\psi_d(t) = \epsilon|R(td)|_1$ which is just the value of the right-hand side of \eqref{eq:V_ge_l1} along the direction $d$. Since both $\phi_d(t), \psi_d(t)$ are piecewise-affine and positive definite, the minimal ratio $\zeta(d) = \min_{t\neq 0} \phi_d(t)/\psi_d(t)$ is strictly positive. Moreover, consider the value $\kappa=\min_{d^Td=1}\zeta(d)$, since the domain $\{d | d^Td=1\}$ is a compact set, and the minimal of a positive function $\zeta(d)$ on a compact set is still positive, hence $\kappa > 0$. As a result, setting $\bar{V}(\cdot) = V(\cdot)/\kappa$ will satisfy \eqref{eq:V_ge_l1}, because $
    \bar{V}(x) = \bar{V}(x^*+td) = V(x^*+td)/\kappa = \phi_d(t)/\kappa \ge \phi_d(t) / \zeta(d) \ge \phi_d(t) / (\phi_d(t)/\psi_d(t)) = \psi_d(t) = \epsilon|R(x-x^*)|_1
$ where we choose $d = (x-x^*)/|x-x^*|$ and $t=|x-x^*|$.
\end{proof}
\subsection{LQR simulation on 3D quadrotor}
\begin{figure}
    \centering
    \includegraphics[width=0.3\textwidth]{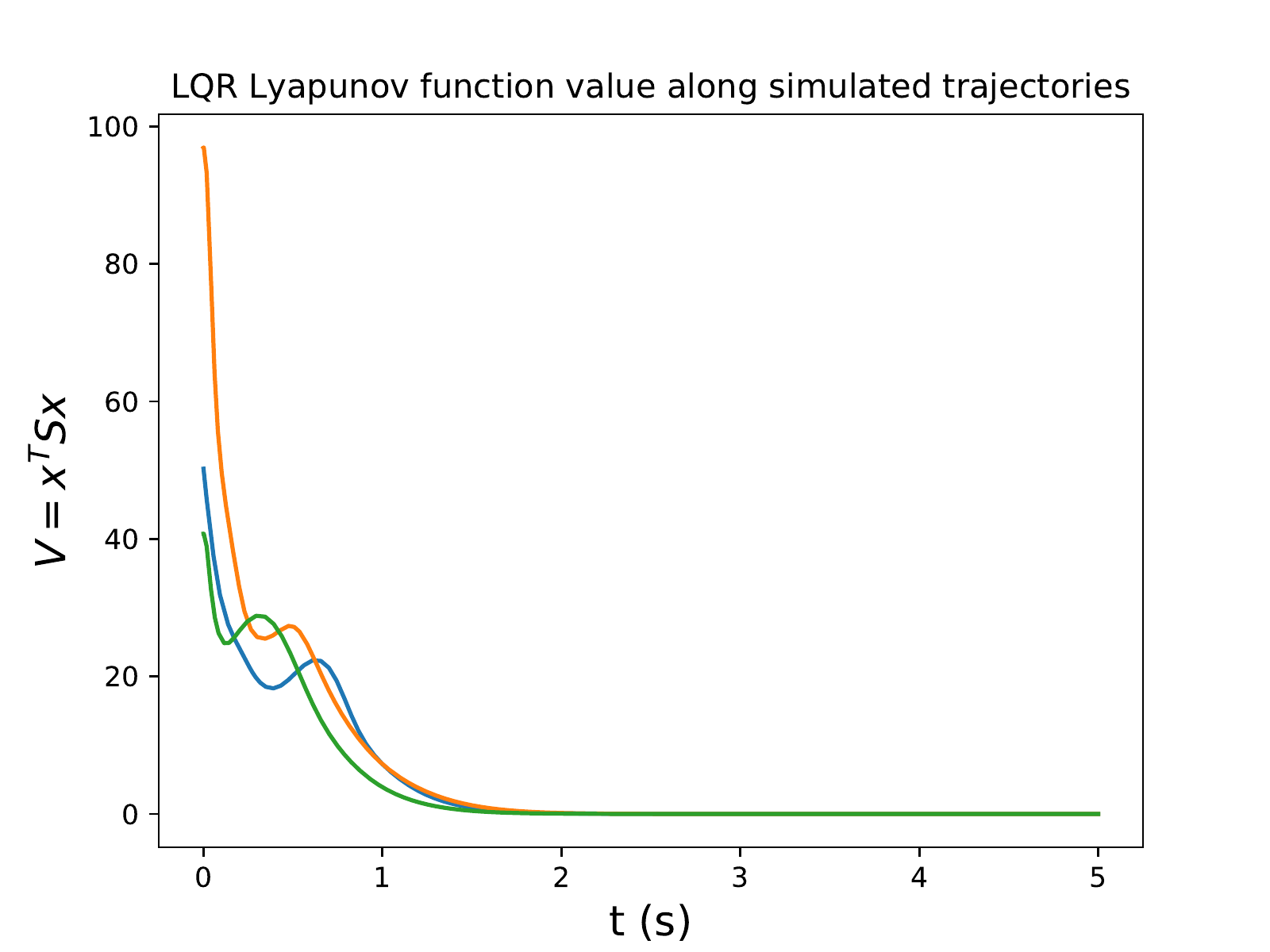}
    \caption{Value of the Lyapunov function $V=x^TSx$ for an LQR controller on a 3D quadrotor.}
    \label{fig:quadrotor3d_lqr_V_traj}
\end{figure}
We simulate the 3D quadrotor with an LQR controller, and plot its Lyapunov function $x^TSx$ (where $S$ is the solution to the Riccati equation) along the simulated trajectories in Fig. \ref{fig:quadrotor3d_lqr_V_traj}. If the dynamics were linear, then this quadratic Lyapunov function would always decrease; on the other hand, with the actual nonlinear dynamics the Lyapunov function (for the linear dynamical system) can increase. For our neural network dynamics, the discrepancy between the approximated dynamics and the actual dynamics will likewise cause the Lyapunov function to increase on the actual dynamical system.

\subsection{Termination tolerance}
In practice, due to solver's numerical tolerance, we declare convergence of Algorithm \ref{algorithm:train_on_sample} and \ref{algorithm:bilevel} when the MIPs \eqref{eq:lyapunov_positivity_mip} \eqref{eq:lyapunov_derivative_mip} has optimal cost in the order of $10^{-6}$.
\end{document}